\documentclass{article}

\usepackage{microtype}
\usepackage{graphicx}
\usepackage{subfigure}
\usepackage{booktabs} 
\usepackage{dirtytalk}

\usepackage{hyperref}

\usepackage[accepted]{icml2023}

\usepackage{amsmath}
\usepackage{amssymb}
\usepackage{mathtools}
\usepackage{amsthm}
\usepackage{lipsum}

\usepackage[capitalize,noabbrev]{cleveref}

\theoremstyle{plain}
\newtheorem{theorem}{Theorem}[section]

\newtheorem{lemma}[theorem]{Lemma}
\newtheorem{corollary}[theorem]{Corollary}
\theoremstyle{definition}
\newtheorem{definition}[theorem]{Definition}

\usepackage[textsize=tiny]{todonotes}

\begin{document}

\twocolumn[
\icmltitle{Stable Autonomous Flow Matching}
\icmlsetsymbol{equal}{*}

\begin{icmlauthorlist}
\icmlauthor{Christopher Iliffe Sprague}{sci,kth}
\icmlauthor{Arne Elofsson}{sci}
\icmlauthor{Hossein Azizpour}{kth}
\end{icmlauthorlist}

\icmlaffiliation{sci}{Department of Biochemistry and Biophysics, Science for Life Laboratory, Stockholm University, Solna, Sweden}
\icmlaffiliation{kth}{Department of Robotics, Perception, and Learning, KTH Royal Institute of Technology, Stockholm Sweden}

\icmlcorrespondingauthor{Christopher Iliffe Sprague}{christopher.iliffe.sprague@gmail.com}


\vskip 0.3in
]



\printAffiliationsAndNotice{}  

\begin{abstract}
    In contexts where data samples represent a physically stable state, it is often assumed that the data points represent the local minima of an energy landscape.
    In control theory, it is well-known that energy can serve as an effective Lyapunov function.
    Despite this, connections between control theory and generative models in the literature are sparse, even though there are several machine learning applications with physically stable data points.
    In this paper, we focus on such data and a recent class of deep generative models called flow matching. We apply tools of stochastic stability for time-independent systems to flow matching models.
    In doing so, we characterize the space of flow matching models that are amenable to this treatment, 
    as well as draw connections to other control theory principles.
    We demonstrate our theoretical results on two examples.
\end{abstract}

\begin{figure}[ht!]
    \centering 
    ~~~~\begin{tabular}{c@{\hspace{1.5cm}}c@{\hspace{1.5cm}}c}
        $t=0.0$ & $t=1.0$ & $t=1.1$ \\
    \end{tabular}
    \raisebox{0.3\height}{\rotatebox{90}{\textbf{Stable-FM}}}
    \includegraphics[width=0.95\linewidth]{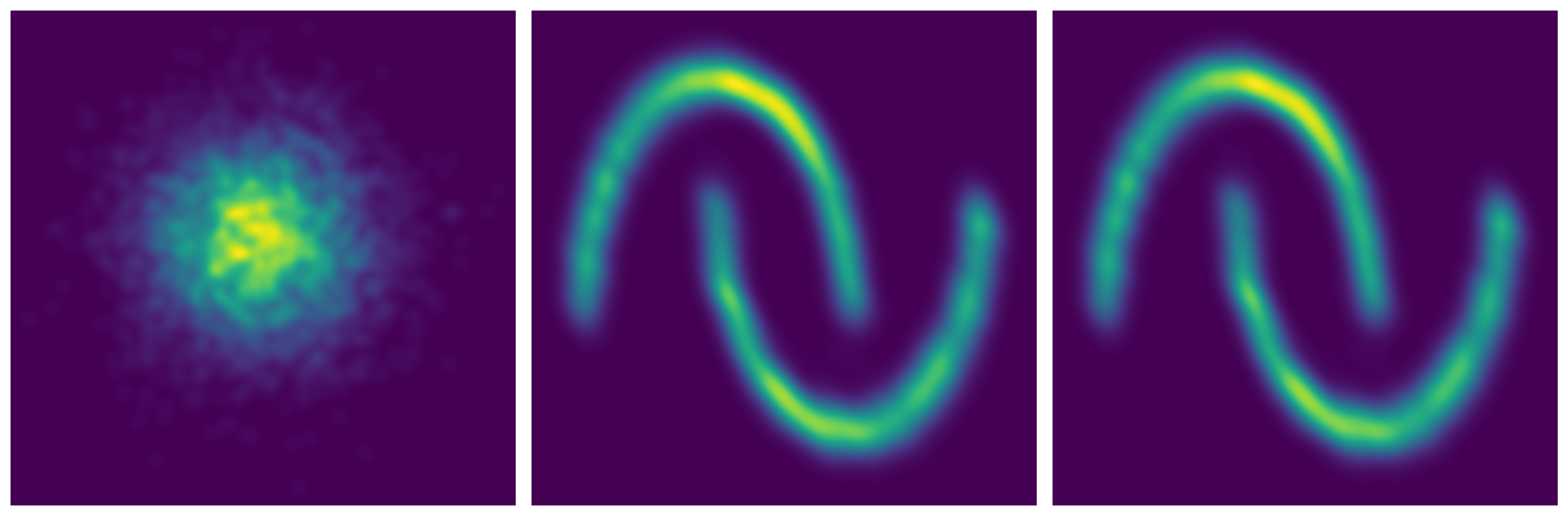}\\ 
    \raisebox{0.5\height}{\rotatebox{90}{\textbf{OT-FM}}} 
    \includegraphics[width=0.95\linewidth]{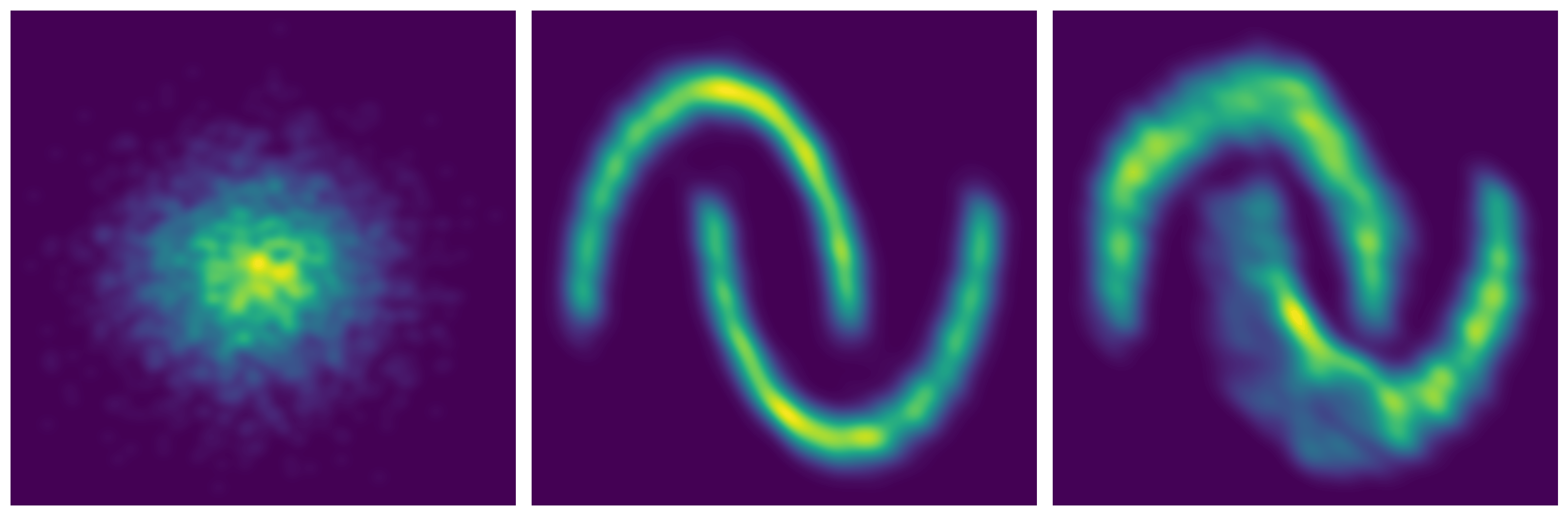}
    \caption{
        Flows of the Stable-FM model (top) and OT-FM model (bottom) from the standard normal distribution to the moons distribution.
        Note that the flow of OT-FM model does not stabilize to the distribution at $t=1$, while the Stable-FM model remains stable to the distribution as $t \to \infty$.
        A depiction with more time steps of the flow and corresponding vector field is shown in \cref{fig:moons_dist_long} and \cref{fig:moons-vecs}, respectively.
    }
    \label{fig:dist-front}
\end{figure}

\section{Introduction}

Generative modeling is a fundamental problem in machine learning, where the goal is usually to learn a model that can generate samples from a target distribution.
In recent years, deep generative models based on continuous-time dynamics \cite{Song2020ScoreBasedGM, Lipman2022FlowMF} have shown exceptional capabilities in a variety of tasks ranging from image generation \cite{Rombach2021HighResolutionIS} to complex structural biology applications \cite{Corso2022DiffDockDS, Ketata2023DiffDockPPRP}.

In contexts where samples from the target distribution represent a physically stable configuration, e.g. molecular conformations \cite{parr1979local}, it is natural to posit that incorporating their physical characteristics into the generative model will bolster performance.
Indeed, a variety of works have attempted to do so through force-field analogies \cite{Shi2021LearningGF, Feng2023MayTF, Luo2021PredictingMC, zaidi2022pre}, which rely on the idea that such samples represent local minima on a free-energy landscape \cite{dill1985theory}.

In control theory, it is well-known that an energy function can serve as an effective Lyapunov function and help to endow a dynamical system with stability \cite{Khalil2002NonlinearSS}.
Despite this, connections in the literature between control theory and generative models for physically stable data are sparse, even though there are existing works on applying Lyapunov stability to machine learning \cite{Kang2021StableNO, Rodriguez2022LyaNetAL, Zhang2022LearningRS, kolter2019learning, Zhang2022NeuralSC, lawrence2020almost}.

To bridge this gap, we apply a stochastic version of La Salle's invariance principle \cite{la1966invariance, MAO1999175} to flow matching (FM) models \cite{Lipman2022FlowMF}, a recent alternative to diffusion models based on continuous normalizing flows (CNFs) \cite{Chen2018NeuralOD}.
The main goal of our work is to equip the dynamics of the FM model with stability to the data samples to respect their intrinsic physical stability.

\paragraph{Contributions}
The main contribution of this work is that we apply the stochastic version of La Salle's invariance principle \cite{la1966invariance, MAO1999175} for \textit{time-independent} systems to FM \cite{Lipman2022FlowMF}. 
In doing so, we contribute the following:
\begin{enumerate}
    \item We show how to construct a time-independent marginal CNF (MCNF) probability density function (PDF) and vector field (VF) in \cref{thm:stationary-mcnf-pdf} and \cref{thm:time-invariant-mcnf-vf}, respectively, and we define the pairs of MCNFs and conditional CNFs (CCNFs) that enable this construction in \cref{def:autonomous-mcnf}.
    Additionally, we define a time-independent conditional FM (CFM) loss in \cref{def:autonomous-fm} and show its gradient's equivalence to the original CFM loss in \cref{thm:autonomous-fm}. We further simplify the loss function in \cref{def:unormalised-autonomous-fm}.
    \item We define the pairs of MCNFs and CCNFs that fulfill the invariance principle of \cref{thm:invariance-principle} in \cref{def:stable-mcnf-space} and \cref{def:stable-mcnf-scalar}. 
    Furthermore, we show that such CCNFs are a superset of the optimal transport CCNF (OT-CCNF) of \cite{Lipman2022FlowMF} under a bijection in \cref{cor:ot-fm}. 
    Additionally, we show that these CCNFs can be designed to converge to the target distribution within a certain time in \cref{cor:convergence-speed}. 
    Lastly, we show the time-independent MCNFs can be viewed as differential inclusion in \cref{cor:inclusion}.
    \item We demonstrate the theoretical results on two datasets, with comparisons to OT-FM.  
\end{enumerate}

\section{Related Work}

In recent years, continuous-time dynamics-based deep generative models \cite{Chen2018NeuralOD,Song2020ScoreBasedGM,Lipman2022FlowMF}
have come to the forefront of the field of deep generative modeling.

Notably, diffusion models \cite{Song2020ScoreBasedGM} were shown to be state-of-the-art in image generation tasks \cite{dhariwal2021diffusion}, and have since
garnered multiple applications ranging from 
structural biology tasks 
    \cite{Corso2022DiffDockDS, Yim2023SE3DM, Ketata2023DiffDockPPRP} to 
video generation 
    \cite{Ho2022VideoDM, Blattmann2023AlignYL, Esser2023StructureAC}, 
with several extensions such as 
latent representations 
    \cite{Vahdat2021ScorebasedGM,Blattmann2023AlignYL} and 
geometric priors 
        \cite{Bortoli2022RiemannianSG,Dockhorn2021ScoreBasedGM}.
Diffusion models are based on the SDEs (e.g. \cite{Srkk2019AppliedSD}) and score matching \cite{Hyvrinen2005EstimationON, Song2019SlicedSM},
where one learns a time-dependent vector-valued score function
$\nabla_\mathbf{x} \log(p(\mathbf{x}, t))$ over diffused data samples and then plugs it into the well-known time-reversed SDE \cite{Anderson1982ReversetimeDE,Haussmann1986TIMERO, Lindquist1979OnTS}.

FM models \cite{Lipman2022FlowMF} were proposed as an alternative to diffusion models \cite{Song2020ScoreBasedGM} that enjoy fast training and sampling, while maintaining competitive performance.
They rely on CNFs \cite{Chen2018NeuralOD} and can be seen as a generalization of diffusion models, as demonstrated by the existence of the probability flow ordinary differential equation (ODE) that induces the same marginal PDF as the SDE of diffusion models \cite{Maoutsa2020InteractingPS, Song2020ScoreBasedGM}.
Several applications of FM models have been made, ranging from structural biology \cite{Yim2023FastPB, Bose2023SE3StochasticFM} to media \cite{Le2023VoiceboxTM, Liu2023I2SBIS}, as well as fundamental extensions \cite{Tong2023ConditionalFM, Pooladian2023MultisampleFM, Shaul2023OnKO, Chen2023RiemannianFM, Klein2023EquivariantFM}.

In the context of data that describes physically stable states, e.g. molecular conformations, several works \cite{zaidi2022pre, Feng2023MayTF, Shi2021LearningGF, Luo2021PredictingMC} have learned force fields by leveraging the connection between the score function and Boltzmann distributions, i.e. $\nabla_\mathbf{x} \log(p(\mathbf{x}, t)) = -\nabla_\mathbf{x} H(\mathbf{x}, t)$ where $H$ is a scalar energy function.
In \cite{zaidi2022pre}, an equivalence between denoising score matching \cite{Vincent2011ACB} and force-field learning is pointed out.
This idea was extended in \cite{Feng2023MayTF} to incorporate off-equilibrium data and NN gradient fields.
In \cite{Shi2021LearningGF, Luo2021PredictingMC}, a NN gradient field is learned to model a psuedo-force field, which is then used to sample energy-minimizing molecular conformations via annealed Langevin dynamics \cite{Song2019GenerativeMB}.

Along the same lines, Poisson Flow Generative Models \cite{Xu2023PFGMUT, Xu2022PoissonFG}, use the solution of Poisson's equation to define a force field as the gradient of the Poisson potential.
Similar to our work, they augment the spatial state with an auxiliary state that acts as an interpolant and replaces time.
However, the gradient field of their spatial state is linked to that of their auxiliary variable.
In contrast, our work deals with quadratic potentials (see \cref{sec:stable-mcnfs}) that yield independent dynamics for the spatial and auxiliary states.
Furthermore, we can govern the rate of convergence using this auxiliary state (see \cref{cor:convergence-speed}).
Additionally, we eliminate the time variable in the context of the conditional FM (CFM) loss \cite{Lipman2022FlowMF}; see \cref{thm:autonomous-fm}.

In control theory, physical stability is often analyzed through Lyapunov stability \cite{Khalil2002NonlinearSS}.
As a result, there are numerous works on applying Lyapunov stability to ODEs/SDEs in the context of dynamics learning \cite{Kang2021StableNO, Rodriguez2022LyaNetAL, Zhang2022LearningRS, kolter2019learning, Zhang2022NeuralSC, lawrence2020almost}.
Despite this, virtually none of the aforementioned works consider stochastic stability to the samples of a target distribution.
In our work, we apply a stochastic invariance principle (\cref{thm:invariance-principle}) to construct stable MCNFs and CCNFs in the context of FM.

\section{Preliminaries}

In this work, we assume that there exists a time-independent PDF $q'(\mathbf{x})$ called a \textit{target PDF} that we do not have direct access to, but we can sample from via a dataset partially contained in its support $\mathcal{X}'$, called the \textit{target space}.
We assume that this support exists in an ambient space $\mathcal{X}$ called the \textit{state space}.

In \cref{sec:flow-matching} we will define CNFs on these spaces, and in \cref{sec:stability} we will define stability on these spaces.
Before continuing, we define some notation.

\paragraph{Notation}
Matrices are denoted in uppercase boldface, vectors in lowercase boldface, and scalars in lowercase non-boldface.
The \emph{state space} is denoted by $\mathcal{X} \subseteq \mathbb{R}^n$ with $n \in \mathbb{N}$, and the \emph{target space} by $\mathcal{X}' \subset \mathcal{X}$. 
The space of continuously differentiable functions from $\mathcal{X}$ to $\mathcal{Y}$ is denoted by $C(\mathcal{X}, \mathcal{Y})$.
The space of non-negative scalars is denoted as $\mathbb{R}_+$.
The space of time-independent PDFs on $\mathcal{X}$ is denoted as $\mathcal{P}(\mathcal{X}) := \{p \in C(\mathcal{X}) \mid \int_\mathcal{X} p(\mathbf{x}) \mathrm{d}\mathbf{x} = 1\}$.
The \textit{target PDF} is denoted as $q' \in \mathcal{P}(\mathcal{X}')$.
The space of time-dependent PDFs on $\mathcal{X}$ is denoted as $\mathcal{P}(\mathcal{X}, \mathbb{R}_+) := \{p \in C(\mathcal{X} \times \mathbb{R}_+) \mid \int_\mathcal{X} p(\mathbf{x}, t) \mathrm{d}\mathbf{x} = 1\}$.
The space of vectors on $\mathcal{X}$ is denoted as $T\mathcal{X}$.
The Dirac-delta PDF centered at $\mathbf{x}'$ is denoted as $\delta_{\mathbf{x}'} \in \mathcal{P}(\mathcal{X})$.
Function-valued functions are denoted s.t. if $f: \mathcal{X} \to C(\mathcal{X}, \mathcal{Y})$ then $f(\cdot \mid \mathbf{x}) \in C(\mathcal{X}, \mathcal{Y})$ and $f(\mathbf{x} \mid \mathbf{x}') \in \mathcal{Y}$.

\subsection{Flow Matching}\label{sec:flow-matching}

To motivate the proposed method, we first review the FM framework \cite{Lipman2022FlowMF}.
The framework relies on the concept of CNFs \cite{Chen2018NeuralOD}, where a simple PDF (e.g. the standard normal PDF) is transformed into a more complex PDF (i.e. the target PDF $q'$) via a VF.
Importantly, the VF must be such that the total PDF mass is conserved, i.e. at every point in time the integral of the PDF over the state space must equal one.
This requirement is achieved when the VF satisfies a continuity equation \cite{Villani2008OptimalTO}.

Formally speaking, a CNF consists of a VF $\mathbf{v}$, a flow map $\mathbf{\psi}$ (i.e. the integral curve of $\mathbf{v}$), and a time-dependent PDF $p$, such that: $(\mathbf{v}, \mathbf{\psi})$ define an ODE and $(\mathbf{v}, p)$ define a continuity equation.
We define the space of CNFs $\mathfrak{F}$ in \cref{def:cnf}.

\begin{definition}[CNF Space]\label{def:cnf}
    A subset
    \begin{equation}
        \mathfrak{F} :\subset C(\mathcal{X} \times \mathbb{R}_{+}, T\mathcal{X}) \times C(\mathcal{X} \times \mathbb{R}_{+}, \mathcal{X}) \times \mathcal{P}(\mathcal{X}, \mathbb{R}_{+})
    \end{equation}
    s.t. for all $(\mathbf{v}, \mathbf{\psi}, p) \in \mathfrak{F}$:
    \begin{align}
        \frac{\mathrm{d} \mathbf{\psi}(\mathbf{x}, t)}{\mathrm{d}t} &= \mathbf{v}(\mathbf{\psi}(\mathbf{x}, t)) \\
        \frac{\partial p(\mathbf{x}, t)}{\partial t} &= -\nabla_\mathbf{x} \cdot \left[\mathbf{v}(\mathbf{x}, t)p(\mathbf{x}, t)\right].
    \end{align}
    We say $(\mathbf{v}, \mathbf{\psi}, p) \in \mathfrak{F}$ is a CNF, $\mathbf{v} \in C(\mathcal{X} \times \mathbb{R}_{+}, T\mathcal{X})$ is a CNF VF, $\mathbf{\psi} \in C(\mathcal{X} \times \mathbb{R}_{+}, \mathcal{X})$ is a CNF flow map, and $p \in \mathcal{P}(\mathcal{X}, \mathbb{R}_{+})$ is a CNF PDF.
\end{definition}

In FM, we assume there exists a latent CNF that describes how a simple PDF (e.g. the standard normal PDF) is transformed into a more complex PDF (i.e. the target PDF $q'$), s.t. $p(\mathbf{x}, T) = q'(\mathbf{x})$ for some $T \in \mathbb{R}_+$.
We then want to train a NN VF $\mathbf{u}_\theta$ to match the latent CNF VF $\mathbf{v}$, via the FM loss in \cref{def:fm-loss}.

\begin{definition}[FM Loss]\label{def:fm-loss}
    A loss function
    \begin{equation}
        L_\text{FM}(\theta) := \underset{\substack{t \sim \mathcal{U}[0, T] \\ \mathbf{x} \sim p(\mathbf{x}, t)}}{\mathbb{E}} \lVert \mathbf{v}_\theta(\mathbf{x}, t) - \mathbf{v}(\mathbf{x}, t) \rVert^2,
    \end{equation}
    given a CNF $(\mathbf{v}, \mathbf{\psi}, p) \in \mathfrak{F}$, a NN VF $\mathbf{v}_\theta \in C(\mathcal{X} \times \mathbb{R}_{+}, T\mathcal{X})$, and a time $T \in \mathbb{R}_+$.
\end{definition}

When $L_\text{FM}(\theta) = 0$, we could then sample $\mathbf{x} \sim p(\mathbf{x}, 0)$ and transform it into $\mathbf{x} \sim q'(\mathbf{x})$ via $\mathbf{\psi}_\theta(\mathbf{x}, T)$, according to the push-forward equation:
\begin{equation}
    p(\mathbf{x}, t) = p(\mathbf{\psi}_\theta^{-1}(\mathbf{x}, t), 0) \det \left( \nabla_\mathbf{x} \mathbf{\psi}_\theta^{-1}(\mathbf{x}, t) \right),
\end{equation}
where $\mathbf{\psi}_\theta^{-1}$ is the inverse of $\psi_\theta$ w.r.t. $\mathbf{x}$, and $\nabla_\mathbf{x} \mathbf{\psi}_\theta^{-1}(\mathbf{x}, t)$ is its Jacobian matrix.

Of course, in all but the simplest cases, we do not have access to the latent CNF, thus $L_\text{FM}$ cannot be used directly.
To ameliorate this issue, FM relies on two key components: CCNFs and MCNFs.
A CCNF is a tractable CNF that resembles the latent CNF's behavior for each sample of the target PDF.
An MCNF is a CNF that is constructed from a mixture of CCNFs in order to match the latent CNF's behavior.
We define the space of CCNFs $\mathfrak{F}'$ and MCNF-CCNF pairs $\mathfrak{F}''$ in \cref{def:ccnf} and \cref{def:mcnf} respectively.  

\begin{definition}[CCNF Space]\label{def:ccnf}
    A subset
    \begin{equation}
        \mathfrak{F}' :\subset C(\mathcal{X}', \mathfrak{F})
    \end{equation}
    s.t. for all $(\mathbf{v}', \mathbf{\psi}', p') \in \mathfrak{F}'$ and $\mathbf{x}' \in \mathcal{X}'$: there exists $T \in \mathbb{R}_+$ s.t. $p'(\mathbf{x}, T \mid \mathbf{x}') \approx \delta_{\mathbf{x}'}(\mathbf{x})$.
    We say $(\mathbf{v}', \mathbf{\psi}', p') \in \mathfrak{F}'$ is a CCNF and 
    $(\mathbf{v}'(\cdot \mid \mathbf{x}'), \mathbf{\psi}'(\cdot \mid \mathbf{x}'), p'(\cdot \mid \mathbf{x}')) \in \mathfrak{F}$ is a CNF given $\mathbf{x}' \in \mathcal{X}'$.
\end{definition}

\begin{definition}[MCNF-CCNF Space]\label{def:mcnf}
    A subset
    \begin{equation}
        \mathfrak{F}'' :\subset \mathfrak{F} \times \mathfrak{F}'
    \end{equation}
    s.t. for all $((\mathbf{v}, \mathbf{\psi}, p), (\mathbf{v}', \mathbf{\psi}', p')) \in \mathfrak{F}''$:
    \begin{align}\label{eq:mcnf-pdf}
        &p(\mathbf{x}, t) := \int_{\mathcal{X}'} p'(\mathbf{x}, t \mid \mathbf{x}')q'(\mathbf{x}')\mathrm{d}\mathbf{x}'.
    \end{align}
    We say $(\mathbf{v}, \mathbf{\psi}, p) \in \mathfrak{F}$ is an MCNF given a CCNF 
    $(\mathbf{v}', \mathbf{\psi}', p') \in \mathfrak{F}'$. 
\end{definition}

The key insight of FM is that since both MCNFs and CCNFs obey the continuity equation in \cref{def:cnf}, the MCNF VF has a specific form, shown in \cref{lem:mvf}.

\begin{lemma}[MCNF VF]\label{lem:mvf}
    ~\\
    For all $((\mathbf{v}, \mathbf{\psi}, p), (\mathbf{v}', \mathbf{\psi}', p')) \in \mathfrak{F}''$:
    \begin{align}\label{eq:mcnf-vector-field}
        &\mathbf{v}(\mathbf{x}, t) = \int_{\mathcal{X}'} \frac{\mathbf{v}'(\mathbf{x}, t \mid \mathbf{x}')p'(\mathbf{x}, t \mid \mathbf{x}')q'(\mathbf{x}')}{p(\mathbf{x}, t)} \mathrm{d}\mathbf{x}',
    \end{align}
    where $p(\mathbf{x}, t)$ is defined in \cref{def:mcnf}.
\end{lemma}
\begin{proof}
    See Theorem 1 of \cite{Lipman2022FlowMF}.
\end{proof}

The MCNF VF in \cref{lem:mvf} leads to the definition of the tractable CFM loss in \cref{def:cfm-loss}, which has identical gradients to the FM loss in \cref{def:fm-loss}, as shown in \cref{thm:cfm}.

\begin{figure*}[ht!]
    \centering
        ~~~~\begin{tabular}{c@{\hspace{1.2cm}}c@{\hspace{1.2cm}}c@{\hspace{1.2cm}}c@{\hspace{1.2cm}}c@{\hspace{1.2cm}}c@{\hspace{1.2cm}}c}
        $t=0.0$ & $t=0.25$ & $t=0.5$ & $t=0.75$ & $t=1.0$ & $t=1.25$ & $t=1.5$ \\
    \end{tabular}\break
    \raisebox{0.2\height}{\rotatebox{90}{\textbf{Stable-FM}}}
    \includegraphics[width=0.95\textwidth]{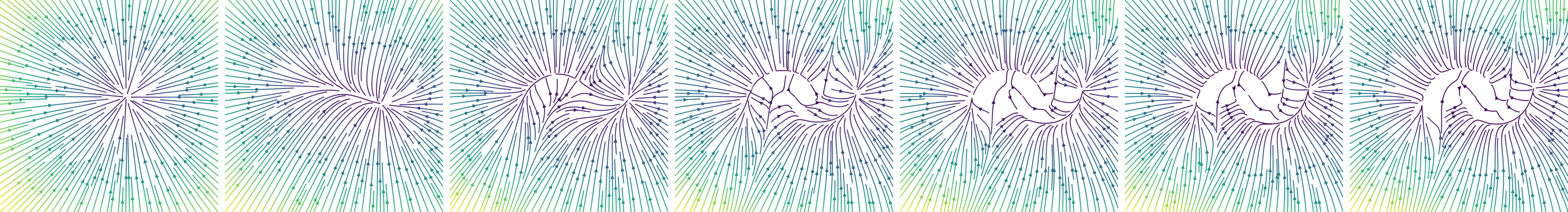} \\
    \raisebox{0.5\height}{\rotatebox{90}{\textbf{OT-FM}}}
    \includegraphics[width=0.95\textwidth]{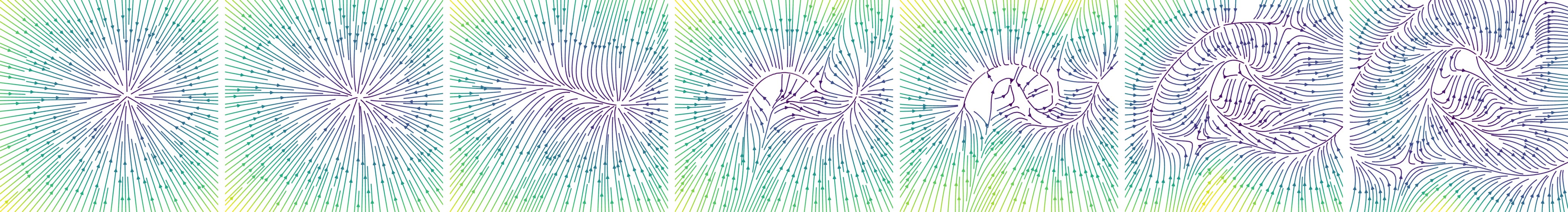}
    \caption{
        Stream plots of the VFs of Stable-FM (top) and OT-FM (bottom) corresponding to \cref{fig:dist-front}.
        Note that beyond $t = 1$, the OT-FM VF diverges, while the Stable-FM VF stabilizes.
        Light colors indicate larger vector magnitudes and v.v..
    }
    \label{fig:moons-vecs}
\end{figure*}

\begin{definition}[CFM Loss]\label{def:cfm-loss}
    A loss function
    \begin{equation}\label{eq:cfm-loss}
        L_{\text{CFM}}(\theta) := \underset{\substack{t \sim \mathcal{U}[0, T] \\ \mathbf{x}' \sim q'(\mathbf{x}')\\ \mathbf{x} \sim p'(\mathbf{x}, t \mid \mathbf{x}')}}{\mathbb{E}} \lVert \mathbf{v}_\theta(\mathbf{x}, t) - \mathbf{v}'(\mathbf{x}, t \mid \mathbf{x}') \rVert^2
    \end{equation}
    given a CCNF $(\mathbf{v}', \mathbf{\psi}', p') \in \mathfrak{F}'$, a NN VF $\mathbf{v}_\theta \in C(\mathcal{X} \times \mathbb{R}_{+}, T\mathcal{X})$, and a time $T \in \mathbb{R}_+$.
\end{definition}

\begin{theorem}[FM and CFM Gradients]\label{thm:cfm}
    For all $((\mathbf{v}, \mathbf{\psi}, p), (\mathbf{v}', \mathbf{\psi}', p')) \in \mathfrak{F}''$:
    \begin{equation}
        \nabla_\theta L_\text{FM}(\theta) = \nabla_\theta L_\text{CFM}(\theta).
    \end{equation}
\end{theorem}
\begin{proof}
    See Theorem 2 of \cite{Lipman2022FlowMF}.
\end{proof}

When $L_\text{CFM}$ is minimized, and assuming $p'(\mathbf{x}, 0 \mid \mathbf{x}') = \mathcal{N}(\mathbf{0}, \mathbf{I})$ for all $\mathbf{x}' \in \mathcal{X}'$, we can then sample $\mathbf{x} \sim \mathcal{N}(\mathbf{0}, \mathbf{I})$ and transform it into $\mathbf{x} \sim q'(\mathbf{x})$ via $\mathbf{\psi}_\theta(\mathbf{x}, T)$, according to the push-forward equation.

\subsection{Stability}\label{sec:stability}

In this work, we want to learn a CNF s.t. the push-forward of the standard normal PDF not only converges to the target PDF $q'$, but also remains stable to its samples.

Stability is a well-studied topic in ODEs \cite{Khalil2002NonlinearSS}, where one typically seeks to characterize stability to a point $\mathbf{x}_\star \in \mathcal{X}$ s.t. $\lim_{t \to \infty} \mathbf{\psi}(\mathbf{x}, t) = \mathbf{x}_\star$ for all starting states $\mathbf{x} \in \mathcal{X}$.
To do this, one can apply Lyapunov's direct method, where if one can find a scalar function $H \in C(\mathcal{X}, \mathbb{R}_+)$ whose gradient along the VF resembles a funnel centered at $\mathbf{x}_\star$, then $\mathbf{x}_\star$ is an asymptotically stable point.

Similarly, one can also characterize stability to a set $\mathcal{X}_\star \subseteq \mathcal{X}$, 
s.t. $\lim_{t \to \infty} \mathbf{\psi}(\mathbf{x}, t) \in \mathcal{X}_\star$ for all $\mathbf{x} \in \mathcal{X}$.
In this case, one can apply La Salle's invariance principle \cite{la1966invariance}, which employs a Lyapunov-like function to locate invariant sets of a system.
Importantly, without further assumptions, this principle only applies to time-independent (autonomous) systems.

Stochastic notions of the aforementioned methods have also been developed in \cite{Kushner1965ONTS,Khasminskii1980StochasticSO,MAO1999175}.
The continuity equation of CNFs in \cref{def:cnf} can be interpreted as the Fokker-Planck-Kolmogorov (FPK) equation (see e.g. Section 5.2 in \cite{Srkk2019AppliedSD}) for an SDE without a diffusion coefficient, thus all of these notions are applicable.

The target distribution $q'$ may be multi-modal, thus we are interested in applying the stochastic version \cite{MAO1999175} of La Salle's invariance principle \cite{la1966invariance}.
We adapt the principle to this setting in \cref{thm:invariance-principle}.

\begin{theorem}[Invariance Principle]\label{thm:invariance-principle}
    If there exists a CNF $(\mathbf{v}, \mathbf{\psi}, p) \in \mathfrak{F}$ s.t. $\mathbf{v}$ is time-independent, and a scalar function $H \in C(\mathcal{X}, \mathbb{R}_+)$ s.t.
    \begin{align}
        &\mathcal{L}_\mathbf{v}H(\mathbf{x}) := \nabla_\mathbf{x} H(\mathbf{x}) \mathbf{v}(\mathbf{x}) \leq 0
    \end{align}
    then
    \begin{equation}
        \lim_{t \to \infty} \mathbf{\psi}(\mathbf{x}, t) \in \left\{\mathbf{x} \in \mathcal{X} \mid \mathcal{L}_\mathbf{v}H(\mathbf{x}) = 0\right\}
    \end{equation}
    almost surely with $\mathbf{x} \sim p(\mathbf{x}, 0)$.
\end{theorem}
\begin{proof}
    See Corollary 4.1 in \cite{MAO1999175}.
\end{proof}

It is well known that finding a Lyapunov-like function $H$ for a given VF $\mathbf{v}$ is generally difficult.
However, if $\mathbf{v}$ is a gradient field of $H$, it is straightforward to show that the principle is satisfied, as shown in \cref{lem:gradient-flow}.

\begin{lemma}[Gradient Fields]\label{lem:gradient-flow}
    If there exists a CNF $(\mathbf{v}, \mathbf{\psi}, p) \in \mathfrak{F}$ and a scalar function $H \in C(\mathcal{X}, \mathbb{R}_+)$ s.t.
    \begin{equation}
        \mathbf{v}(\mathbf{x}) = -\nabla_\mathbf{x} H(\mathbf{x})^\top,
    \end{equation}
    then \cref{thm:invariance-principle} holds.
\end{lemma}
\begin{proof}
    $\mathcal{L}_\mathbf{v}H(\mathbf{x}) = -\nabla_\mathbf{x} H(\mathbf{x}) \nabla_\mathbf{x} H(\mathbf{x})^\top \leq 0$.
\end{proof}

Aside from enforcing stability, gradient fields are also unique in the context of the continuity equation due to the well-known Helmholtz decomposition, see \cite{ambrosio2005gradient, neklyudov2022action, bhatia2012helmholtz, richter2022neural}, which states that a VF can be constructed by the sum of a curl-free VF (e.g. a gradient field) and divergence-free VF.
It is easy to see that adding a divergence-free component does not affect the behavior of the continuity equations in \cref{def:cnf}.

\section{Main Result}\label{sec:main-result}

\subsection{Autonomous MCNFs and CCNFs}\label{sec:autonomous-mcnfs}

Our goal is to construct a CNF  s.t. we can apply the invariance principle of \cref{thm:invariance-principle} to guarantee its stability to the support of the target distribution $q'$.
Unfortunately, even if the CCNF VF used to construct the MCNF VF is time-independent, the MCNF VF will not be time-independent due to its dependence on the CCNF's time-dependent PDF $p'(x, t \mid x')$ in \cref{lem:mvf}.

To ameliorate this, we consider that time in the context of MCNFs is merely used as a parameter of interpolation between the initial and target PDFs.
Thus, we can instead consider how the MCNF behaves w.r.t. to one of the variables in the state space.

To explore this line of reasoning, we augment the state space with a new scalar state $\tau$, which can be thought of as a pseudo-time or temperature.
We then consider that this state evolves deterministically over time via a Dirac-delta PDF, but is still involved as an argument to the VFs.
We define this new MCNF-CCNF space with a data space $\mathcal{Z}$ and a pseudo-time space $\mathcal{T}$ in \cref{def:autonomous-mcnf}

\begin{definition}[Auto MCNF-CCNF Space]\label{def:autonomous-mcnf}
    A subset 
    \begin{equation}
        \mathfrak{F}_\text{Auto}'' :\subset \mathfrak{F}'' \quad \text{s.t.} \quad 
        \begin{aligned}
            &\mathcal{X} := \mathcal{Z} \times \mathcal{T}, \quad \mathcal{T} \subset \mathbb{R} \\
            &q'(\mathbf{x}) := q_\mathbf{z}'(\mathbf{z})q_\tau'(\tau)
        \end{aligned}
    \end{equation}
    and for all $((\mathbf{v}, \mathbf{\psi}, p), (\mathbf{v}', \mathbf{\psi}', p')) \in \mathfrak{F}''_\text{Auto}$:
    \begin{align}
        \mathbf{v}'(\mathbf{x} \mid \mathbf{x}') =& \begin{bmatrix}
            \mathbf{v}_\mathbf{z}'(\mathbf{z} \mid \mathbf{z}') \\
            v_\tau'(\tau \mid \tau')
        \end{bmatrix} \\
        \mathbf{\psi}'(\mathbf{x}, t \mid \mathbf{x}') =& \begin{bmatrix}
            \mathbf{\psi}_\mathbf{z}'(\mathbf{z}, t \mid \mathbf{z}') \\
            \psi_\tau'(\tau, t \mid \tau')
        \end{bmatrix} \\
        \label{eq:autonomous-mcnf-pdf}
        p'(\mathbf{x}, t \mid \mathbf{x}') =& p'_\mathbf{z}(\mathbf{z}, t \mid \mathbf{z}') p_\tau'(\tau, t \mid \tau').
    \end{align}
    Furthermore, with $\tau_0 \in \mathcal{T}$ and $\tau_1 \in \mathcal{T}'$ s.t. $\tau_0 \neq \tau_1$:
    \begin{equation}\label{eq:autonomous-mcnf-pdfs}
        q_\tau'(\tau) = \delta_{\tau_1}(\tau), \quad p'_\tau(\tau, t \mid \tau') = \delta_{\psi_\tau'(\tau_0, t \mid \tau')}(\tau).
    \end{equation}
    We say $(\mathbf{v}, \mathbf{\psi}, p) \in \mathfrak{F}$ is an autonomous MCNF given an autonomous CCNF $(\mathbf{v}', \mathbf{\psi}', p') \in \mathfrak{F}'$.
\end{definition}

The key part of \cref{def:autonomous-mcnf} is that the target PDF $q'_\tau$ over $\mathcal{T}'$ is a dirac-delta PDF centered at $\tau_1$, and the CCNF's PDF $p'_\tau$ over $\mathcal{T}$ is dirac-delta PDF centered at $\psi_\tau'(\tau_0, t \mid \tau')$.
Thus, $\tau$ will evolve deterministically over time from $\tau_0$ to $\tau_1$.
Now, assuming that CCNF flow map $\psi_\tau'$ is bijective with time, we can then use its inverse to map $\tau$ to $t$ and render the MCNF autonomous, as we show in \cref{thm:stationary-mcnf-pdf} and \cref{thm:time-invariant-mcnf-vf}.
Henceforth, we use $\bar{\cdot}$ to denote the substitution of $t$ with $\mathbf{\psi}_\tau'^{-1}(\tau_0, \tau \mid \tau_1)$.

\begin{theorem}[Time-Independent MCNF PDF]\label{thm:stationary-mcnf-pdf}
    For all $((\mathbf{v}, \mathbf{\psi}, p), (\mathbf{v}', \mathbf{\psi}', p')) \in \mathfrak{F}_\text{Auto}''$,
    there exists a time-independent PDF $\bar{p} \in \mathcal{P}(\mathcal{X})$ corresponding to $p$:
    \begin{equation}
        \bar{p}(\mathbf{x}) := p(\mathbf{z}, \tau, \psi_\tau'^{-1}(\tau_0, \tau \mid \tau_1)).
    \end{equation}
\end{theorem}
\begin{proof}
    See \cref{thma:stationary-mcnf-pdf}.
\end{proof}

\begin{theorem}[Time-Independent MCNF VF]\label{thm:time-invariant-mcnf-vf}
    For all $((\mathbf{v}, \mathbf{\psi}, p), (\mathbf{v}', \mathbf{\psi}', p')) \in \mathfrak{F}_\text{Auto}''$,
    there exists a time-independent VF $\bar{v} \in C(\mathcal{X}, T\mathcal{X})$ corresponding to $\bar{p}$.
\end{theorem}
\begin{proof}
    See \cref{thma:time-invariant-mcnf-vf}.  
\end{proof}

Now, in the same way that we have eliminated time from the MCNF PDF in \cref{thm:stationary-mcnf-pdf} and the MCNF VF in \cref{thm:time-invariant-mcnf-vf}, we now attempt to eliminate it from the CFM loss in \cref{def:cfm-loss} to obtain a time-independent way of training in \cref{def:autonomous-fm}.

\begin{definition}[Auto CFM Loss]\label{def:autonomous-fm}
    A loss function
    \begin{equation}
        L_{\text{Auto}}(\theta) := \underset{\substack{\tau \sim \mathcal{U}[\tau_0, \tau_1] \\ \mathbf{z}' \sim q_\mathbf{z}'(\mathbf{z}') \\ \mathbf{z} \sim \bar{p}_\mathbf{z}'(\mathbf{z}, \tau \mid \mathbf{z}')}}{\mathbb{E}} \frac{\lVert \mathbf{v}_\theta(\mathbf{z}, \tau) - \mathbf{v}'(\mathbf{z}, \tau \mid \mathbf{z}', \tau_1) \rVert^2}{v_\tau'(\tau \mid \tau_1)},
    \end{equation}
    given an autonomous CCNF $(\mathbf{v}', \mathbf{\psi}', p') \in \mathfrak{F}'$, a NN VF $\mathbf{v}_\theta \in C(\mathcal{X} \times \mathbb{R}_{+}, T\mathcal{X})$, and a time $T \in \mathbb{R}_+$ s.t. $\tau_0 = \psi_\tau'(\tau_0, 0 \mid \tau_1)$ and $\tau_1 = \psi_\tau'(\tau_0, T \mid \tau_1)$, where
    $\bar{p}_\mathbf{z}'(\mathbf{z}, \tau) = p_\mathbf{z}'(\mathbf{z}, \tau, \psi_\tau'^{-1}(\tau_0, \tau \mid \tau_1))$.
\end{definition}

The autonomous CFM loss in \cref{def:autonomous-fm} was identified in a similar way to the time-independent MCNF PDF and VF in \cref{thm:stationary-mcnf-pdf} and \cref{thm:time-invariant-mcnf-vf}. By using the deterministic evolution of $\tau$ in addition to its bijective flow map $\psi_\tau'$, the CFM loss can be rendered time-independent.
As a result, they share the same gradients, as stated in \cref{thm:autonomous-fm}.

\begin{theorem}[FM, CFM, and Auto CFM Gradients]\label{thm:autonomous-fm}
    For all $((\mathbf{v}, \mathbf{\psi}, p), (\mathbf{v}', \mathbf{\psi}', p')) \in \mathfrak{F}_\text{Auto}''$:
    \begin{equation}
        \nabla_\theta L_{\text{FM}}(\theta) = \nabla_\theta L_{\text{CFM}}(\theta) = \nabla_\theta L_{\text{Auto}}(\theta).
    \end{equation}
\end{theorem}
\begin{proof}
    See \cref{thma:autonomous-fm}.
\end{proof}

While the auto CFM loss $L_{\text{Auto}}$ in \cref{def:autonomous-fm} is effective, there may be issues with the normalization term $v_\tau'(\tau \mid \tau_1)$ in the denominator, particularly in the case where we want $\mathbf{v}(\mathbf{z}', \tau' \mid \mathbf{z}', \tau') = 0$, which is the case if we desire stability (see \cref{sec:stability}).
One straightforward way to relieve this issue is to instead sample with $\tau \sim \mathcal{U}[\tau_0, \tau_1 - \epsilon]$ for some small $\epsilon \in \mathbb{R}_+$.
However, since we are interpreting the MCNF VF as a time-independent, we can just take the expectation over $\tau$ from the start without employing integration by substitution of the time argument, as we do in \cref{def:unormalised-autonomous-fm}.

\begin{definition}[Unormalized Auto CFM Loss]\label{def:unormalised-autonomous-fm}
    A loss function
    \begin{equation}
        L_{\text{Auto}}'(\theta) = \underset{\substack{\tau \sim \mathcal{U}[\tau_0, \tau_1] \\ \mathbf{z}' \sim q_\mathbf{z}'(\mathbf{z}') \\ \mathbf{z} \sim p_\mathbf{z}'(\mathbf{z}, \tau \mid \mathbf{z}')}}{\mathbb{E}} \lVert \mathbf{v}_\theta(\mathbf{z}, \tau) - \mathbf{v}(\mathbf{z}, \tau \mid \mathbf{z}', \tau_1) \rVert^2
    \end{equation}
    given an autonomous CCNF $(\mathbf{v}', \mathbf{\psi}', p') \in \mathfrak{F}'$, a NN VF $\mathbf{v}_\theta \in C(\mathcal{X} \times \mathbb{R}_{+}, T\mathcal{X})$, and a time $T \in \mathbb{R}_+$ s.t. $\tau_0 = \psi_\tau'(\tau_0, 0 \mid \tau_1)$ and $\tau_1 = \psi_\tau'(\tau_0, T \mid \tau_1)$.
\end{definition}

Importantly, the unormalized auto CFM loss $L_\text{Auto}'$ in \cref{def:unormalised-autonomous-fm} does not exhibit undefined behavior when $\tau = \tau_1$.
In the next section, we will consider $\tau = \tau_1$ as representing a stable state on the support of the target distribution $q'$.

\subsection{Stable MCNFs and CCNFs}\label{sec:stable-mcnfs}

Now that we have constructed an autonomous MCNF VF, we can enforce the remaining stability conditions in \cref{thm:invariance-principle}, namely that there exists some scalar function $H$ s.t. $\mathcal{L}_\mathbf{v}H(\mathbf{x}) \leq 0$.
To do this, we need to consider the properties of both the NN VF we want to learn and the MCNF VF.

First, we can enforce the stability of the NN VF in the sense of \cref{thm:invariance-principle} in a straightforward way by modeling it as a gradient field, as in \cref{lem:gradient-flow}.
We define this NN vector field in \cref{def:gradient-cnf}.

\begin{definition}[NN Gradient Field]\label{def:gradient-cnf}
    An autonomous CNF $(\mathbf{v}_\theta, \mathbf{\psi}_\theta, p_\theta) \in \mathfrak{F}$
    with state space $\mathcal{X} := \mathcal{Z} \times \mathcal{T}$ s.t. $\mathcal{T} \subset \mathbb{R}$, a scalar function $H_\theta \in C(\mathcal{X}, \mathbb{R}_+)$, and a VF
    \begin{equation}
        \mathbf{v}_\theta(\mathbf{z}, \tau) := -\nabla_{(\mathbf{z}, \tau)} H_\theta(\mathbf{z}, \tau)^\top.
    \end{equation}
\end{definition}

While the NN VF in \cref{def:gradient-cnf} is stable to the subset of the state space where $\nabla_{(\mathbf{z}, \tau)} H(\mathbf{z}, \tau) = 0$ (i.e. the set of all local minima of $H$), it is not guaranteed that this subset lies in the support of the target distribution $q'$.
Therefore, we need to ensure that the CCNF VF used to train the NN VF is stable to the samples of the target distribution $q'$.
We define the space of MCNF-CCNF pairs that permit this in \cref{def:stable-mcnf-space}.

\begin{definition}[Stable MCNF-CCNF Space]\label{def:stable-mcnf-space}
    A subset
    \begin{equation}
        \mathfrak{F}_\text{Stable}'' :\subset \mathfrak{F}_\text{Auto}'' \quad \text{s.t.} \quad \exists H': \mathcal{X}' \to C(\mathcal{X}, \mathbb{R}_+)
    \end{equation}
    s.t. for all $((\mathbf{v}, \mathbf{\psi}, p), (\mathbf{v}', \mathbf{\psi}', p')) \in \mathfrak{F}_\text{Stable}''$:
    \begin{align}
        H'(\mathbf{x} \mid \mathbf{x}') =& \frac{1}{2}(\mathbf{x} - \mathbf{x}')^\top \mathbf{A} (\mathbf{x} - \mathbf{x}') \\
        \mathbf{v}'(\mathbf{x} \mid \mathbf{x}') =& -\nabla_{\mathbf{x}} H'(\mathbf{x} \mid \mathbf{x}')^\top = -\mathbf{A}(\mathbf{x} - \mathbf{x}') 
    \end{align}
    where $\mathbf{A} \in \mathbb{R}_+^{\mathrm{dim}(\mathcal{X}) \times \mathrm{dim}(\mathcal{X})}$ is positive definite.
\end{definition}

Importantly, the CCNF VF $\mathbf{v}'$ in \cref{def:stable-mcnf-space} is a gradient field of a quadratic potential function $H'(\cdot \mid \mathbf{x}')$, which means that it is exponentially stable to a single point $\mathbf{x}'$ (i.e. the global minimum of $H'$).
Moreover, it is a linear VF that grants a flow map $\psi'$ and PDF $p'$ in a tractable form.
However, the flow map and PDF involve the matrix exponential (see Section 6.2 of \cite{Srkk2019AppliedSD}), e.g.
$$
\begin{aligned}
\mathbf{\psi}'(\mathbf{x}, t \mid \mathbf{x}') =& \mathbf{x}' + \exp\left(-\mathbf{A}t\right)(\mathbf{x} - \mathbf{x}'), 
\end{aligned}
$$
which can be difficult to compute in practice \cite{van1977sensitivity}, unless $\mathbf{A}$ is either diagonalizable, diagonal, or scalar.
In this work, we limit our attention to a semi-scalar case, where the sub-matrices corresponding to $\mathbf{z}$ and $\tau$ are diagonal with identical eigenvalues; the other cases will be left to future work\footnote{We could also consider different eigenvalues for each sample.}.
We define this MCNF-CCNF space in \cref{def:stable-mcnf-scalar}.

\begin{definition}[Scalar Stable MCNF-CCNF Space]\label{def:stable-mcnf-scalar}
    A subset
    \begin{equation}
        \mathfrak{F}_\text{S-Stable}'' :\subset \mathfrak{F}_\text{Stable}'' \quad \text{s.t.} \quad         \mathbf{A} = \begin{bmatrix}
            \lambda_\mathbf{z} \mathbf{I}_{\mathrm{dim}(\mathcal{Z})} & \mathbf{0} \\
            \mathbf{0} & \lambda_\tau
        \end{bmatrix}
    \end{equation}
    and $\lambda_\mathbf{z}, \lambda_\tau \in \mathbb{R}_+$, where $\mathbf{I}_{\mathrm{dim}(\mathcal{Z})}$ is the identity matrix.
\end{definition}

\begin{figure}
    \centering
    \includegraphics[width=0.92\linewidth]{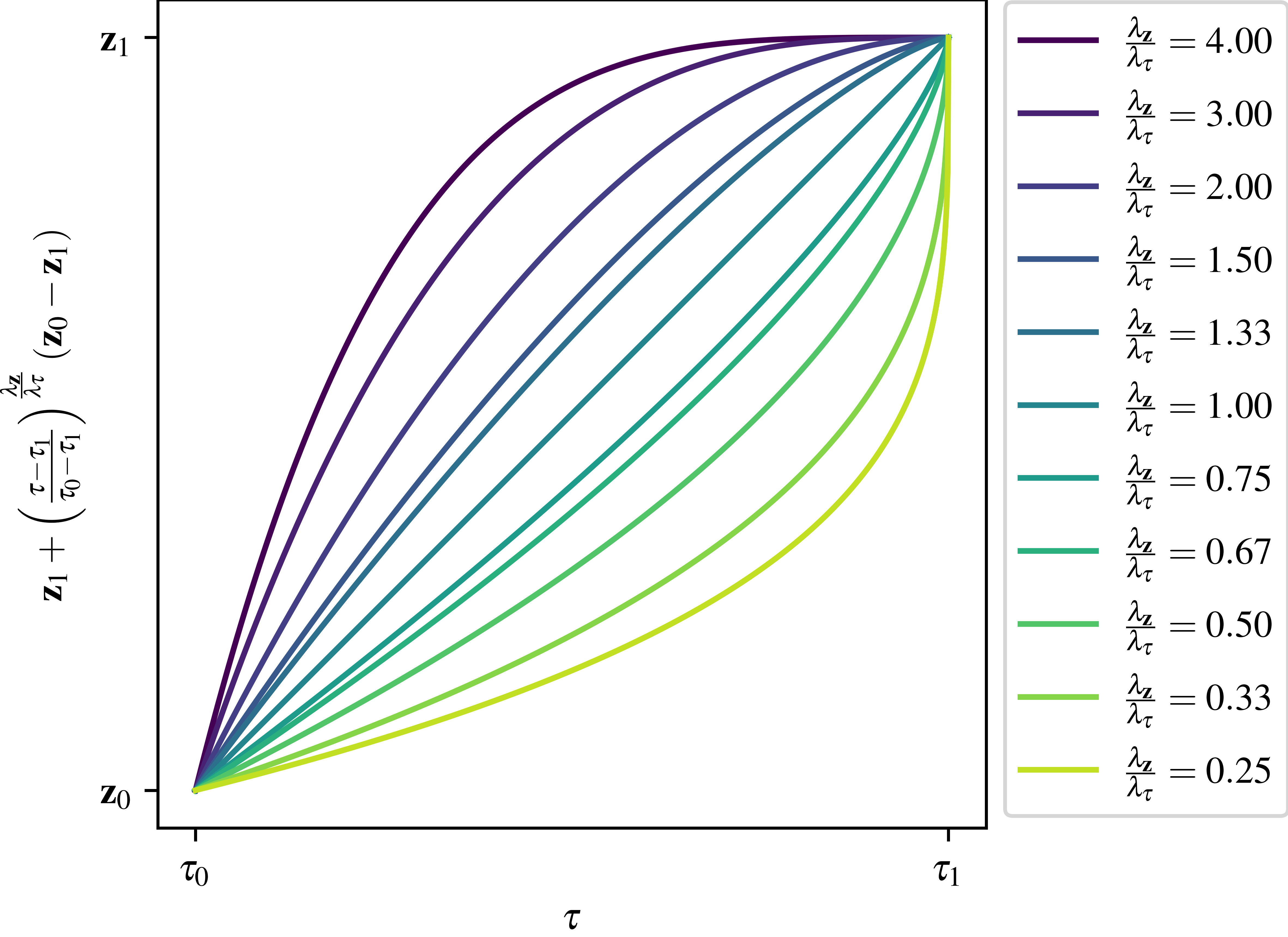}
    \caption{
        A plot of the scalar stable-CCNF flow map $\mathbf{\psi}_\mathbf{z}'(\mathbf{z}, \psi_\tau'^{-1}(\tau_0, \tau \mid \tau_1) \mid \mathbf{z}')$ over values of $\tau \in [\tau_0, \tau_1]$ for different values of $\lambda_\mathbf{z}$ and $\lambda_\tau$.
        Note that $\frac{\lambda_\mathbf{z}}{\lambda_\tau} = 1$ corresponds to the OT-CCNF flow map as shown in \cref{cor:ot-fm}.
    }
    \label{fig:interpolation}
\end{figure}

We will now explicitly state the CCNF VF, flow map, and PDF corresponding to a scalar stable MCNF-CCNF space in \cref{lem:scalar-stable-mcnf-components}.

\begin{lemma}[Scalar Stable CCNF Components]\label{lem:scalar-stable-mcnf-components}
    For all $((\mathbf{v}, \mathbf{\psi}, p), (\mathbf{v}', \mathbf{\psi}', p')) \in \mathfrak{F}_\text{S-Stable}''$:
    \begin{align}
        \mathbf{v}'(\mathbf{x} \mid \mathbf{x}') 
        &= \begin{bmatrix}
            -\lambda_\mathbf{z}(\mathbf{z} - \mathbf{z}') \\
            -\lambda_\tau(\tau - \tau')
        \end{bmatrix} \\
        \mathbf{\psi}'(\mathbf{x}, t \mid \mathbf{x}')
        &= \begin{bmatrix}
            \mathbf{z}' + \exp\left(-\lambda_\mathbf{z}t\right)(\mathbf{z} - \mathbf{z}') \\
            \tau' + \exp\left(-\lambda_\tau t\right)(\tau - \tau')
        \end{bmatrix} \\
        p'(\mathbf{x}, t \mid \mathbf{x}') &= \mathcal{N}\left(
        \mathbf{\mu}'(t \mid \mathbf{x}'), 
        \mathbf{\Sigma}'(t \mid \mathbf{x}')\right) \\
        \mathbf{\mu}'(t \mid \mathbf{x}') &= 
        \begin{bmatrix}
            \mathbf{z}' + \exp\left(-\lambda_\mathbf{z}t\right)(\mathbf{z}_0 - \mathbf{z}') \\
            \tau' + \exp\left(-\lambda_\tau t\right)(\tau_0 - \tau')
        \end{bmatrix} \\
        \mathbf{\Sigma}'(t \mid \mathbf{x}') &= \begin{bmatrix}
            \exp(-2 \lambda_\mathbf{z} t) \mathbf{\Sigma}_0 & \mathbf{0} \\
            \mathbf{0} & 0
        \end{bmatrix},
    \end{align}
    where $\mathbf{z}_0 \in \mathcal{Z}$ 
    and $\mathbf{\Sigma}_0 = \mathbb{R}^{\mathrm{dim}(\mathcal{Z}) \times \mathrm{dim}(\mathcal{Z})}$ are the initial mean and covariance 
    of $\mathbf{z}$, and $\tau_0 \in \mathcal{T}$ and $0$ are the initial mean and variance of $\tau$.
\end{lemma}
\begin{proof}
    A straightforward application of Section 6.2. of \cite{Srkk2019AppliedSD} with a zero diffusion coefficient.
\end{proof}

Now that we have seen how the components of the scalar stable CCNF look in \cref{lem:scalar-stable-mcnf-components}, it is interesting to see how this affects the loss functions $L_\text{Auto}$ and $L_\text{Auto}'$.
Both of these loss functions involve an expectation over the CCNF PDF $\bar{p}_\mathbf{z}'$.
Interestingly, this CCNF PDF yields an interpolation between $\mathbf{z}_0$ and $\mathbf{z_1}$ for which we can control the rate of interpolation by varying $\lambda_\mathbf{z}$ and $\lambda_\tau$, as shown in \cref{lem:scalar-stable-interpolation} shown in \cref{fig:interpolation}.

\begin{lemma}[Scalar Stable CCNF PDF]\label{lem:scalar-stable-interpolation}
    For all $((\mathbf{v}, \mathbf{\psi}, p), (\mathbf{v}', \mathbf{\psi}', p')) \in \mathfrak{F}_\text{S-Stable}''$:
    \begin{align}
        \bar{p}_\mathbf{z}'(\mathbf{z}, \tau \mid \mathbf{z}') =& \mathcal{N}\left(\bar{\mathbf{\mu}}_\mathbf{z}(\tau \mid \mathbf{z}'), \bar{\mathbf{\Sigma}}_\mathbf{z}(\tau \mid \mathbf{z}')\right) \\
        \bar{\mathbf{\mu}}_\mathbf{z}(\tau \mid \mathbf{z}') =& \mathbf{z}' + \left(\frac{\tau - \tau_1}{\tau_0 - \tau_1}\right)^{\frac{\lambda_\mathbf{z}}{\lambda_\tau}}(\mathbf{z}_0 - \mathbf{z}') \\
        \bar{\mathbf{\Sigma}}_\mathbf{z}(t \mid \mathbf{z}') =& \left(\frac{\tau - \tau_1}{\tau_0 - \tau_1}\right)^{\frac{2\lambda_\mathbf{z}}{\lambda_\tau}} \mathbf{\Sigma}_0.
    \end{align}
\end{lemma}

An interesting thing to note from \cref{lem:scalar-stable-interpolation} is that the choice of $\lambda_\mathbf{z} = \lambda_\tau$, shown by the linear curve in \cref{fig:interpolation}, makes the mean $\bar{\mathbf{\mu}}_\mathbf{z}$ of $\mathbf{z}$ look similar to the OT-CCNF flow map of \cite{Lipman2022FlowMF} when we choose $\tau_0 = 0$ and $\tau_1 = 1$, as shown in \cref{cor:ot-fm}.

\begin{corollary}[Scalar Stable-CCNF vs. OT-CCNF]\label{cor:ot-fm}
    For all $((\mathbf{v}, \mathbf{\psi}, p), (\mathbf{v}', \mathbf{\psi}', p')) \in \mathfrak{F}_\text{S-Stable}''$:
    if $\tau_0 = 0$, $\tau_1 = 1$, and $\lambda_\mathbf{z} = \lambda_\tau$, then $\bar{\psi}_\mathbf{z}'(\mathbf{z}, \tau \mid \mathbf{z}')$ and $\bar{\mathbf{v}}_\mathbf{z}'(\mathbf{z}, \tau \mid \mathbf{z}')$ are equivalent to the OT-CCNF flow map and VF of Example 2 in \cite{Lipman2022FlowMF} with $t \gets \tau$, $\mathbf{x} \gets \mathbf{z}$, $\mathbf{x}_1 \gets \mathbf{z}'$, $\sigma_\text{min} \gets 0$.
\end{corollary}
\begin{proof}
    The OT-CCNF flow map and VF of \cite{Lipman2022FlowMF} are given by
    \begin{align}
        \mathbf{\psi}_\text{OT}(\mathbf{x}, t \mid \mathbf{x}_1) =& (1 - (1 - \sigma_\text{min})t)\mathbf{x} + t\mathbf{x}_1 \\
        \mathbf{v}_\text{OT}(\mathbf{x}, t \mid \mathbf{x}_1) =& \frac{\mathbf{x}_1 - (1 - \sigma_\text{min}) \mathbf{x}}{1 - (1 - \sigma_\text{min})t}.
    \end{align}
    The scalar stable-CCNF flow map and VF as a function of $\tau$ are given with the substitutions as 
    \begin{align}
        \bar{\mathbf{\psi}}_\mathbf{z}'(\mathbf{z}, \tau \mid \mathbf{z}') =& \mathbf{z}' + \left(\frac{\tau - \tau_1}{\tau_0 - \tau_1}\right)^{\frac{\lambda_\mathbf{z}}{\lambda_\tau}}(\mathbf{z} - \mathbf{z}') \\
        =& (1 - \tau) \mathbf{z} + \tau \mathbf{z}' \nonumber \\
        \bar{\mathbf{v}}_\mathbf{z}'(\mathbf{z}, \tau \mid \mathbf{z}') =& \frac{\mathbf{v}_\mathbf{z}'(\mathbf{z} \mid \mathbf{z}')}{v_\tau'(\tau \mid \tau_1)} = \frac{\mathbf{z}' - \mathbf{z}}{1 - \tau},
    \end{align}
    where $\bar{\mathbf{v}}_\mathbf{z}'$ is found through the chain rule $\frac{\partial \mathbf{z}}{\partial \tau} = \frac{\partial \mathbf{z}}{\partial t}\frac{\partial t}{\partial \tau}$ and $\bar{\mathbf{\psi}}_\mathbf{z}'$ is its solution from $\tau_0$ to $\tau_1$.
\end{proof}

Note that in the OT-CCNF VF of \cite{Lipman2022FlowMF}, the added $\sigma_\text{min}$ term in the denominator helps to avoid undefined values when $\tau = 1$.
The same issue is encountered in the $L_{\text{Auto}}$ loss in \cref{def:autonomous-fm}, where the normalization term $v_\tau'(\tau \mid \tau_1)$ is undefined when $\tau = \tau_1$.
The solution proposed in \cref{sec:autonomous-mcnfs} of sampling with $\tau \sim \mathcal{U}[\tau_0, \tau_1 - \epsilon]$ for some small $\epsilon \in \mathbb{R}_+$ is equivalent to including the $\sigma_\text{min}$ term of the OT-CCNF.

One problem, however, is that if $\epsilon$ is too small then the loss will be dominated by the $\tau = \tau_1$ samples.
On the other hand, if $\epsilon$ is too large then the samples will not capture the stabilizing behavior of the VF near $\tau = \tau_1$.
We can address this issues either by using $L_\text{Auto}'$ (\cref{def:unormalised-autonomous-fm}) or by selecting $\frac{\lambda_\mathbf{z}}{\lambda_\tau} > 1$ (see \cref{fig:interpolation}), which will cause $\mathbf{z}$ to converge to $\mathbf{z}_1$ faster than $\tau$ converges to $\tau_1$, relieving the need for a larger $\epsilon$.
While the ratio of $\lambda_\mathbf{z}$ and $\lambda_\tau$ affects the rate of interpolation in the $\mathcal{T}$ domain, one can also affect the rate interpolation in the time domain, as demonstrated in \cref{cor:convergence-speed}.

\begin{corollary}[Scalar Stable-CCNF Convergence Rate]\label{cor:convergence-speed}
    For all $((\mathbf{v}, \mathbf{\psi}, p), (\mathbf{v}', \mathbf{\psi}', p')) \in \mathfrak{F}_\text{S-Stable}''$: if there exists $(T, \epsilon_\tau) \in \mathbb{R}_+^2$ s.t.
    \begin{equation}
        \lambda_\tau \geq -\ln\left(\frac{\epsilon_\tau}{|\tau_0 - \tau_1|}\right)\frac{1}{T}
    \end{equation}
    then $\left|\psi_\tau'(\tau_0, T \mid \tau_1) - \tau_1\right| \leq \epsilon_\tau$.
    Furthermore, if there exists $(T, \epsilon_\mathbf{z}) \in \mathbb{R}_+^2$ s.t.
    \begin{equation}
        \lambda_\mathbf{z} \geq -\ln\left(\frac{\epsilon_\tau}{\lVert\mathbf{z}_0 - \mathbf{z}_1\rVert}\right) \frac{1}{T}
    \end{equation}
    then $\lVert\psi_\mathbf{z}'(\mathbf{z}_0, T \mid \mathbf{z}_1) - \mathbf{z}_1\rVert \leq \epsilon_\mathbf{z}$.
\end{corollary}

Note that $\lambda_\mathbf{z}$ and $\lambda_\tau$ must be positive for $\mathbf{A}$ to be positive-definite.
Thus, considering that $\lVert\mathbf{z}_0 - \mathbf{z}_1\rVert < \epsilon_\mathbf{z}$ may not always be true since we will be sampling $\mathbf{x}_0$, and that we will always start with $\tau_0 = 0$ and $\tau_1 = 1$, it is more practical to select $\lambda_\tau$ first and then select $\lambda_\mathbf{z}$ according to the ratio $\frac{\lambda_\mathbf{z}}{\lambda_\tau}$ we want (see \cref{fig:interpolation}).
E.g., if we want $\tau$ to be within $0.1$ distance from $\tau_1 = 1$ at $t=1$ starting from $\tau_0 = 0$, we can select $\lambda_\tau = \ln(0.1)$ and $\lambda_\mathbf{z}$ accordingly.

To close off the main result, we would like to state an interesting observation that links the stability of autonomous MCNF-CCNF pairs to the topic of differential inclusions in \cref{cor:inclusion}.

\begin{corollary}[MCNF Differential Inclusion]\label{cor:inclusion}
    For all $((\mathbf{v}, \mathbf{\psi}, p), (\mathbf{v}', \mathbf{\psi}', p')) \in \mathfrak{F}_\text{Auto}''$:
    \begin{equation}
        \frac{\mathrm{d} \mathbf{\psi}(\mathbf{x}, t)}{\mathrm{d} t} = \mathrm{co}\left\{\mathbf{v}'(\mathbf{x} \mid \mathbf{x}') \mid \mathbf{x}' \in \mathcal{X}'\right\},
    \end{equation}
    where $\mathrm{co}$ is the convex hull operator.
\end{corollary}
\begin{proof}
    Looking at \cref{lem:mvf}, the statement is evident from the fact that
    \begin{equation}
        \int_{\mathcal{X}'}\frac{p'(\mathbf{x}, t \mid \mathbf{x}')q'(\mathbf{x}')\mathrm{d}\mathbf{x}'}{p(\mathbf{x}, t)} = 1.
    \end{equation}
\end{proof}

Differential inclusions \cite{filippov1960differential} occur frequently in the hybrid dynamical systems literature \cite{goebel2009hybrid, sanfelice2021hybrid}.
E.g. in switched systems \cite{liberzon1999basic}, it represents arbitrary switching between different VFs.
In discontinuous dynamical systems \cite{cortes2008discontinuous}, it represents sliding dynamics between VFs.
Stability results for differential inclusions have been explored in \cite{boyd1994linear,molchanov1989criteria}, and in \cite{hafez2022stability, veer2019switched,alpcan2010stability} when multiple equilibria are present.
These results are highly relevant as an MCNF VF is a convex combination of linear VFs induced by the CCNF.

\subsection{Experiments}

We explored the efficacy of Stable-FM on the moons and circles distributions (setup detailed in \cref{seca:experiments}).
In \cref{fig:dist-front}, \cref{fig:moons-vecs}, and \cref{fig:moons_dist_long} we demonstrate how Stable-FM generates flows that are stable to the support of the target distribution, while OT-FM does not.
Looking at \cref{fig:moons-vecs}, it is evident that the VF of Stable-FM converges to a stable one, while the VF of OT-FM diverges abruptly after $t=1$.
On the circles distribution, we trained the Stable-FM model for various values of $\frac{\lambda_\mathbf{z}}{\lambda_\mathbf{\tau}}$, corresponding to the curves in \cref{fig:interpolation}, with most models converging, see \cref{fig:stable-circles-dist}.
But, we found that with $\frac{\lambda_\mathbf{z}}{\lambda_\tau} = 1$, which corresponds to OT-CFM (see \cref{cor:ot-fm}), the model did not converge.
We suspect that this has something to do with the scale of the convergence rate relative to the scale of the distance between local minima (the distance between circles in this case).
We see that higher values of $\frac{\lambda_\mathbf{z}}{\lambda_\tau}$ lead to a sharper change in the vector field, see \cref{fig:stable-circles-vecs}.
Overall, we found that our theoretical results are validated.

\section{Conclusion}

In this paper we applied a stochastic version of La Salle's invariance principle \cite{la1966invariance, MAO1999175} to FM \cite{Lipman2022FlowMF} to enforce stability of the model's flows to the support of the target distribution, a desirable property when the data represents a physically stable state.
In doing so, we showed how to render the MCNF VF time-independent by introducing a psuedo-time $\tau$ and performing interpolation over it instead of time.
Enforcing this time-independence was necessary to apply the invariance principle.
Additionally, we showed how to construct CCNF VFs that satisfy the invariance principle and we showed how they and their flow map compare to the OT-CCNF VF and flow map from \cite{Lipman2022FlowMF}.
Lastly, we have made several connections from control theory to generative modeling.
Overall, we demonstrated that our approach is effective with theoretical and experimental results.

\section{Acknowledgments}
We thank Ruibo Tu for insightful discussions.
This work was partially supported by the Wallenberg AI, Autonomous Systems
and Software Program (WASP) funded by the Knut and Alice Wallenberg Foundation.

\section{Impact Statement}
This paper presents work whose goal is to advance the field of Machine Learning. There are many potential societal consequences of our work, none of which we feel must be specifically highlighted here.

\bibliography{refs.bib}
\bibliographystyle{icml2023}

\newpage
\appendix
\onecolumn
\section{Autonomous Flow Matching}

\begin{theorem}[Time-Independent MCNF PDF]\label{thma:stationary-mcnf-pdf}
    For all $((\mathbf{v}, \mathbf{\psi}, p), (\mathbf{v}', \mathbf{\psi}', p')) \in \mathfrak{F}_\text{Auto}''$,
    there exists a stationary PDF 
    \begin{equation}
        \bar{p}(\mathbf{x}) := p(\mathbf{z}, \tau, \psi_\tau'^{-1}(\tau_0, \tau \mid \tau_1)).
    \end{equation}
\end{theorem}
\begin{proof}
    Starting from the MCNF PDF \cref{eq:mcnf-pdf} in \cref{def:mcnf} and noting the PDFs \cref{eq:autonomous-mcnf-pdfs} in \cref{def:autonomous-mcnf}, we have:
    \begin{align}\label{eq:autonomous-mcnf-pdf-simplified}
        p(\mathbf{x}, t) &\overset{(i)}{=} \int_{\mathcal{X}'} p'(\mathbf{x}, t \mid \mathbf{x}') q'(\mathbf{x}') \mathrm{d}\mathbf{x}' \\
        &\overset{(ii)}{=} \int_{\mathcal{Z}' \times \mathcal{T}} p'_\mathbf{z}(\mathbf{z}, t \mid \mathbf{z}') p'_\tau(\tau, t \mid \tau') q_\mathbf{z}'(\mathbf{z}') q_\tau'(\tau') \mathrm{d}(\mathbf{z}', \tau') \nonumber \\
        &\overset{(iii)}{=} \left(\int_{\mathcal{Z}'} p_\mathbf{z}'(\mathbf{z}, t \mid \mathbf{z}') q_\mathbf{z}'(\mathbf{z}') \mathrm{d}\mathbf{z}'\right) \left(\int_{\mathcal{T}'} p'_\tau(\tau, t \mid \tau') q_\tau'(\tau') \mathrm{d}\tau'\right) \nonumber \\
        &\overset{(iv)}{=} p_\mathbf{z}(\mathbf{z}, t) \int_{\mathcal{T}'} \delta_{\psi_\tau'(\tau_0, t \mid \tau')}(\tau) \delta_{\tau_1}(\tau') \mathrm{d}\tau' \nonumber \\
        &\overset{(v)}{=} p_\mathbf{z}(\mathbf{z}, t) \delta_{\psi_\tau'(\tau_0, t \mid \tau_1)}(\tau), \nonumber
    \end{align}
    where in (i) we have the definition \cref{eq:mcnf-pdf}, in (ii) we substituted \cref{eq:autonomous-mcnf-pdf}, in (iii) we used the fact that $q_\mathbf{z}'(\mathbf{z}')$ and $q_\tau'(\tau')$ are independent, in (iv) we substituted \cref{eq:autonomous-mcnf-pdfs} and defined $p_\mathbf{z}(\mathbf{z}, t):= \int_{\mathcal{Z}'} p_\mathbf{z}'(\mathbf{z}, t \mid \mathbf{z}') q_\mathbf{z}'(\mathbf{z}') \mathrm{d}\mathbf{z}'$, and in (v) we used the property of the Dirac delta PDF.

    Using the fact that $\psi_\tau'$ is invertible w.r.t. $t$, i.e. s.t. 
    $\psi_\tau'(\tau_0, \psi_\tau'^{-1}(\tau_0, \tau, \mid \tau_1) \mid \tau_1) = \tau$ and 
    $\psi_\tau'^{-1}(\tau_0, \psi_\tau'(\tau_0, t, \mid \tau_1) \mid \tau_1) = t$, we apply a change of variables to obtain:
    \begin{align}
        \bar{p}(\mathbf{x}) \overset{(i)}{:=}& p(\mathbf{z}, \tau, \psi_\tau'^{-1}(\tau_0, \tau \mid \tau_1)) 
        \det\left(
        \begin{bmatrix}
            \frac{\partial \mathbf{z}}{\partial \mathbf{z}} & \frac{\partial \mathbf{z}}{\partial \tau} & \frac{\partial \mathbf{z}}{\partial t} \\
            \frac{\partial \tau}{\partial \mathbf{z}} & \frac{\partial \tau}{\partial \tau} & \frac{\partial \tau}{\partial t} \\
            \frac{\partial \psi_\tau'^{-1}(\tau_0, \tau \mid \tau_1)}{\partial \mathbf{z}} & \frac{\partial \psi_\tau'^{-1}(\tau_0, \tau \mid \tau_1)}{\partial \tau} & \frac{\partial \psi_\tau'^{-1}(\tau_0, \tau \mid \tau_1)}{\partial \psi_\tau'^{-1}(\tau_0, \tau \mid \tau_1)}
        \end{bmatrix}
        \right) \\
        \overset{(ii)}{=}& p(\mathbf{z}, \tau, \psi_\tau'^{-1}(\tau_0, \tau \mid \tau_1)) 
        \det\left(
        \begin{bmatrix}
            \mathbf{1} & \mathbf{0} & \mathbf{0} \\
            0 & 1 & 0 \\
            0 & \frac{\partial \psi_\tau'^{-1}(\tau_0, \tau \mid \tau_1)}{\partial \tau} & 1
        \end{bmatrix}
        \right) \nonumber \\
        \overset{(iii)}{=}& p(\mathbf{z}, \tau, \psi_\tau'^{-1}(\tau_0, \tau \mid \tau_1)) \det \left(
            \begin{bmatrix}
                1 & 0 \\
                \frac{\partial \psi_\tau'^{-1}(\tau_0, \tau \mid \tau_1)}{\partial \tau} & 1
            \end{bmatrix}
        \right) \nonumber \\
        \overset{(iv)}{=}& p(\mathbf{z}, \tau, \psi_\tau'^{-1}(\tau_0, \tau \mid \tau_1)) \nonumber \\
        \overset{(v)}{=}& p_\mathbf{z}(\mathbf{z}, \psi_\tau'^{-1}(\tau_0, \tau \mid \tau_1)) \delta_{\psi_\tau'(\tau_0, \psi_\tau'^{-1}(\tau_0, \tau, \mid \tau_1) \mid \tau_1)}(\tau) \nonumber \\
        \overset{(vi)}{=}& p_\mathbf{z}(\mathbf{z}, \psi_\tau'^{-1}(\tau_0, \tau \mid \tau_1)), \nonumber
    \end{align}
    where in (i) we substitute $t \gets \psi_\tau'^{-1}(\tau_0, t \mid \tau_1)$ and apply the change of variables formula, in (ii)-(iv) we simplify the Jacobian determinant to $1$, in (v) we substitute the simplified MCNF PDF of \cref{eq:autonomous-mcnf-pdf-simplified}, in (vi)-(vi) we use the properties of the dirac-delta PDF.
\end{proof}

\begin{theorem}[Time-Independent MCNF VF]\label{thma:time-invariant-mcnf-vf}
    For all $((\mathbf{v}, \mathbf{\psi}, p), (\mathbf{v}', \mathbf{\psi}', p')) \in \mathfrak{F}_\text{Auto}''$,
    there exists a time-independent vector field $\bar{\mathbf{v}} \in C(\mathcal{X}, T\mathcal{X})$ corresponding to $\bar{p}$:
    \begin{equation}
        \bar{\mathbf{v}}(\mathbf{x}) := \int_{\mathcal{Z}'} \frac{\mathbf{v}'(\mathbf{z}, \tau \mid \mathbf{z}', \tau_1) p_\mathbf{z}'(\mathbf{z}, \psi_\tau'^{-1}(\tau_0, \tau \mid \tau_1) \mid \mathbf{z}') q_\mathbf{z}'(\mathbf{z}')}{\bar{p}(\mathbf{x})}\mathrm{d}\mathbf{z}'.
    \end{equation}
\end{theorem}
\begin{proof}
    Starting from the MCNF vector field of \cref{eq:mcnf-vector-field} in \cref{lem:mvf} and noting the PDFs of \cref{eq:autonomous-mcnf-pdfs} in \cref{def:autonomous-mcnf}, we have:
    \begin{align}\label{eq:autonomous-mcnf-vf-simplified}
        \mathbf{v}(\mathbf{x}, t) \overset{(i)}{=}& \int_{\mathcal{X}'} \frac{\mathbf{v}'(\mathbf{x} \mid \mathbf{x}')p'(\mathbf{x}, t \mid \mathbf{x}')q'(\mathbf{x}')}{p(\mathbf{x}, t)} \mathrm{d}\mathbf{x}' \\
        \overset{(ii)}{=}& \frac{1}{p(\mathbf{z}, \tau, t)} \int_{\mathcal{Z}' \times \mathcal{T}'} \mathbf{v}'(\mathbf{z}, \tau \mid \mathbf{z}', \tau') p_\mathbf{z}'(\mathbf{z}, t \mid \mathbf{z }') q_\mathbf{z}'(\mathbf{z}') p_\tau'(\tau, t \mid \tau') q_\tau'(\tau') \mathrm{d}(\mathbf{z}', \tau') \nonumber \\
        \overset{(iii)}{=}& \frac{1}{p(\mathbf{z}, \tau, t)} \int_{\mathcal{Z}' \times \mathcal{T}'} \mathbf{v}'(\mathbf{z}, \tau \mid \mathbf{z}', \tau') p_\mathbf{z}'(\mathbf{z}, t \mid \mathbf{z}') q_\mathbf{z}'(\mathbf{z}') \delta_{\psi_\tau'(\tau_0, t \mid \tau')}(\tau) \delta_{\tau_1}(\tau') \mathrm{d}(\mathbf{z}', \tau') \nonumber \\
        \overset{(iv)}{=}& \frac{1}{p(\mathbf{z}, \tau, t)} \int_{\mathcal{Z}'} \mathbf{v}'(\mathbf{z}, \tau \mid \mathbf{z}', \tau_1) p_\mathbf{z}'(\mathbf{z}, t \mid \mathbf{z}') q_\mathbf{z}'(\mathbf{z}') \delta_{\psi_\tau'(\tau_0, t \mid \tau_1)}(\tau) \mathrm{d}\mathbf{z}' \nonumber,
    \end{align}
    where in (i) we substitute the definition of \cref{eq:mcnf-vector-field}, in (ii)-(iii) we substitute the PDFs of \cref{eq:autonomous-mcnf-pdfs} in \cref{def:autonomous-mcnf}, in (iv) we use the properties of the dirac-delta PDF.

    Now, using the fact that $\psi_\tau'$ is invertible w.r.t. $t$, as before, we apply a change of variables to obtain:
    \begin{align}
        \bar{\mathbf{v}}(\mathbf{x}) \overset{(i)}{:=}& \mathbf{v}(\mathbf{z}, \tau, \psi_\tau'^{-1}(\tau_0, \tau \mid \tau_1)) \\
        \overset{(ii)}{=}& \frac{1}{p_\mathbf{z}(\mathbf{z}, \tau, \psi_\tau'^{-1}(\tau_0, \tau \mid \tau_1))} \int_{\mathcal{Z}'} \mathbf{v}'(\mathbf{z}, \tau \mid \mathbf{z}', \tau_1) p_\mathbf{z}'(\mathbf{z}, \psi_\tau'^{-1}(\tau_0, \tau \mid \tau_1) \mid \mathbf{z}') q_\mathbf{z}'(\mathbf{z}') \delta_{\tau}(\tau) \mathrm{d}\mathbf{z}' \nonumber \\
        \overset{(iii)}{=}& \frac{1}{\bar{p}(\mathbf{x})} \int_{\mathcal{Z}'} \mathbf{v}'(\mathbf{z}, \tau \mid \mathbf{z}', \tau_1) p_\mathbf{z}'(\mathbf{z}, \psi_\tau'^{-1}(\tau_0, \tau \mid \tau_1) \mid \mathbf{z}') q_\mathbf{z}'(\mathbf{z}') \mathrm{d}\mathbf{z}' \nonumber,
    \end{align}
    where in (i) we substitute $t \gets \psi_\tau'^{-1}(\tau_0, t \mid \tau_1)$ and apply the change of variables formula as before, in (ii) we expand according to \cref{eq:autonomous-mcnf-vf-simplified}, and in (iii) we use the properties of the dirac-delta PDF.
\end{proof}

\begin{theorem}[Autonomous FM Loss]\label{thma:autonomous-fm}
    For all $((\mathbf{v}, \mathbf{\psi}, p), (\mathbf{v}', \mathbf{\psi}', p')) \in \mathfrak{F}_\text{Auto}''$:
    \begin{equation}
        \nabla_\theta L_{\text{FM}}(\theta) = \nabla_\theta L_{\text{CFM}}(\theta) = \nabla_\theta L_{\text{Auto}}(\theta).
    \end{equation}
\end{theorem}
\begin{proof}
    It suffices to show that $L_\text{Auto}$ can be obtained from $L_\text{CFM}$ by the bijection $\psi_\tau'(\tau_0, \tau \mid \tau_1)$.
    Starting with the CFM loss of \cref{eq:cfm-loss} of \cref{def:cfm-loss}:
    \begin{align}
        L_{\text{CFM}}(\theta) \overset{(i)}{=}& \underset{\substack{t \sim \mathcal{U}[0, T] \\ \mathbf{x}' \sim q'(\mathbf{x}')\\ \mathbf{x} \sim p'(\mathbf{x}, t \mid \mathbf{x}')}}{\mathbb{E}} \lVert \mathbf{v}_\theta(\mathbf{x}, t) - \mathbf{v}'(\mathbf{x} \mid \mathbf{x}') \rVert^2 \\
        \overset{(ii)}{=}& \int_0^T \int_{\mathcal{X}'} \int_{\mathcal{X}} \lVert \mathbf{v}_\theta(\mathbf{x}, t) - \mathbf{v}'(\mathbf{x} \mid \mathbf{x}') \rVert^2 p'(\mathbf{x}, t \mid \mathbf{x}') q'(\mathbf{x}')
        \mathrm{d}\mathbf{x}\mathrm{d}\mathbf{x}'\mathrm{d}t \nonumber \\
        \overset{(iii)}{=}& \int_0^T \int_{\mathcal{Z}' \times \mathcal{T}'} \int_{\mathcal{Z} \times \mathcal{T}} \lVert \mathbf{v}_\theta(\mathbf{z}, \tau, t) - \mathbf{v}'(\mathbf{z}, \tau \mid \mathbf{z}', \tau') \rVert^2 p_\mathbf{z}'(\mathbf{z}, t \mid \mathbf{z}') q_\mathbf{z}'(\mathbf{z}') p_\tau'(\tau, t \mid \tau') q_\tau'(\tau')
        \mathrm{d}(\mathbf{z}, \tau)\mathrm{d}(\mathbf{z}', \tau') \mathrm{d}t \nonumber \\
        \overset{(iv)}{=}& \int_0^T \int_{\mathcal{Z} \times \mathcal{T}} \int_{\mathcal{Z}' \times \mathcal{T}'} \lVert \mathbf{v}_\theta(\mathbf{z}, \tau, t) - \mathbf{v}'(\mathbf{z}, \tau \mid \mathbf{z}', \tau') \rVert^2 p_\mathbf{z}'(\mathbf{z}, t \mid \mathbf{z}') q_\mathbf{z}'(\mathbf{z}') 
        \delta_{\psi_\tau'(\tau_0, t \mid \tau')}(\tau) \delta_{\tau_1}(\tau')
        \mathrm{d}(\mathbf{z}, \tau)\mathrm{d}(\mathbf{z}', \tau') \mathrm{d}t \nonumber \\
        \overset{(v)}{=}& \int_0^T  \int_{\mathcal{Z} \times \mathcal{T}} \int_{\mathcal{Z}'}\lVert \mathbf{v}_\theta(\mathbf{z}, \tau, t) - \mathbf{v}'(\mathbf{z}, \tau \mid \mathbf{z}', \tau_1) \rVert^2 p_\mathbf{z}'(\mathbf{z}, t \mid \mathbf{z}') q_\mathbf{z}'(\mathbf{z}') 
        \delta_{\psi_\tau'(\tau_0, t \mid \tau_1)}(\tau)
        \mathrm{d}\mathbf{z}' \mathrm{d}(\mathbf{z}, \tau) \mathrm{d}t \nonumber \\
        \overset{(vi)}{=}& \int_{\tau_0}^{\tau_1} \int_{\mathcal{Z}'} \int_{\mathcal{Z}} 
        \frac{\lVert \mathbf{v}_\theta(\mathbf{z}, \tau, \psi_\tau'^{-1}(\tau_0, \tau \mid \tau_1)) - \mathbf{v}'(\mathbf{z}, \tau \mid \mathbf{z}', \tau_1) \rVert^2 p_\mathbf{z}'(\mathbf{z}, \psi_\tau'^{-1}(\tau_0, \tau \mid \tau_1) \mid \mathbf{z}') q_\mathbf{z}'(\mathbf{z}')}{v_\tau'(\tau \mid \tau_1)}
        \mathrm{d}(\mathbf{z}, \tau) \mathrm{d}\mathbf{z}'\mathrm{d}\tau \nonumber \\
        \overset{(vii)}{=}& \underset{\substack{\tau \sim \mathcal{U}[\tau_0, \tau_1] \\ \mathbf{z}' \sim q_\mathbf{z}'(\mathbf{z}') \\ \mathbf{z} \sim \bar{p}_\mathbf{z}'(\mathbf{z}, \tau \mid \mathbf{z}')}}{\mathbb{E}} \frac{\lVert \mathbf{v}_\theta(\mathbf{z}, \tau, \psi_\tau'^{-1}(\tau_0, \tau \mid \tau_1)) - \mathbf{v}'(\mathbf{z}, \tau \mid \mathbf{z}', \tau_1) \rVert^2}{v_\tau'(\tau \mid \tau_1)},
    \end{align}
    where in (i) we state the definition from \cref{eq:cfm-loss} of \cref{def:cfm-loss}, in (ii) we write the expression integral form, in (iii) we expand the expression according to \cref{def:autonomous-mcnf}, in (iv) we substitute the autonomous MCNF PDFs from \cref{eq:autonomous-mcnf-pdfs} and change the integral order (by Fubini's theorem), in (v) we simplify the integral with the properties of the dirac-delta PDF, in (vi) apply integration by substitution with $\tau_0 = \psi_\tau'(\tau_0, 0 \mid \tau_1)$ and $\tau_1 = \psi_\tau'(\tau_0, T, \tau_1)$ and the fact that $\frac{\mathrm{d} \tau}{\mathrm{d} t} = v_\tau'(\tau \mid \tau')$ and again exchange the integral order.
    The equivalence $\nabla_\theta L_\text{FM} = \nabla_\theta L_\text{Auto}$ then follows from $\nabla_\theta L_\text{FM} = \nabla_\theta L_\text{CFM}$ in \cref{thm:cfm}.
    The NN vector field may just be simplified to $\mathbf{v}_\theta(\mathbf{z}, \tau)$.
\end{proof}

\begin{lemma}[Scalar Stable Auto FM Loss]\label{lem:losses-explicit}
    For all NN gradient field and $((\mathbf{v}, \mathbf{\psi}, p), (\mathbf{v}', \mathbf{\psi}', p')) \in \mathfrak{F}_\text{S-Stable}''$:
    \begin{align}
        L_{\text{Auto}}(\theta) = 
        \underset{\substack{\tau \sim \mathcal{U}[\tau_0, \tau_1] \\ \mathbf{z}' \sim q_\mathbf{z}'(\mathbf{z}') \\ \mathbf{z} \sim \bar{p}_\mathbf{z}'(\mathbf{z}, \tau \mid \mathbf{z}')}}{\mathbb{E}} 
        \left\lVert \nabla_{(\mathbf{z}, \tau)} H_\theta(\mathbf{z}, \tau) -
        \begin{bmatrix}
            -\lambda_\mathbf{z} (\mathbf{z} - \mathbf{z}') \\
            -\lambda_\tau (\tau - \tau_1)
        \end{bmatrix}
        \right\rVert^2 \frac{1}{-\lambda_\tau (\tau - \tau_1)},
    \end{align}
    and $L_\text{Auto}'$ is found by removing the denominator in $L_\text{Auto}$.
\end{lemma}
\begin{proof}
    A straightforward application of \cref{def:autonomous-fm} and \cref{lem:scalar-stable-mcnf-components}.
\end{proof}

\section{Experiments}\label{seca:experiments}

    We explore the efficacy of our approach on the moons and circles dataset from Scikit-Learn, each with $0.05$ standard deviation of noise and $100000$ samples for training.
    We used JAX, Optax, and Flax for our models and training.
    For OT-FM we use a dense network $\mathbf{v}_\theta: \mathbb{R}^2 \to \mathbb{R}^2$.
    For Stable-FM we use a dense network $H_\theta: \mathbb{R}^3 \to \mathbb{R}_+$ s.t. $\mathbf{v}_\theta (\mathbf{z}, \tau) = -\nabla_{(\mathbf{z}, \tau)} H_\theta(\mathbf{z}, \tau)$.
    For both networks, we use $4$ hidden layers, each with $500$ nodes and softplus activation (also on the output for Stable-FM for positivity).
    For all figures shown, we train with the Adam optimizer with weight decay regularization for $20000$ iterations and a batch size of $10000$ with a learning rate of $0.001$ and all other parameters as default.
    For all experiments, we use $L_\text{Auto}'$.

\begin{figure}[ht!]
    \centering
        ~~~~\begin{tabular}{c@{\hspace{1.2cm}}c@{\hspace{1.2cm}}c@{\hspace{1.2cm}}c@{\hspace{1.2cm}}c@{\hspace{1.2cm}}c@{\hspace{1.2cm}}c}
        $t=0.0$ & $t=0.25$ & $t=0.5$ & $t=0.75$ & $t=1.0$ & $t=1.25$ & $t=1.5$ \\
    \end{tabular}\break
    \raisebox{0.2\height}{\rotatebox{90}{\textbf{Stable-FM}}}
    \includegraphics[width=0.95\linewidth]{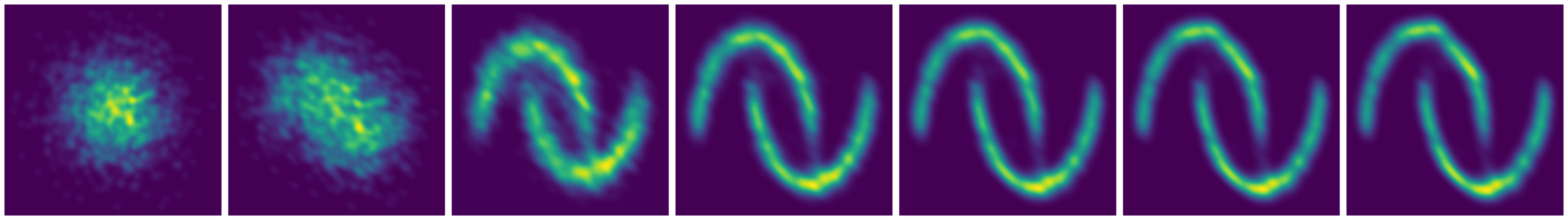} \\
    \raisebox{0.5\height}{\rotatebox{90}{\textbf{OT-FM}}}
    \includegraphics[width=0.95\linewidth]{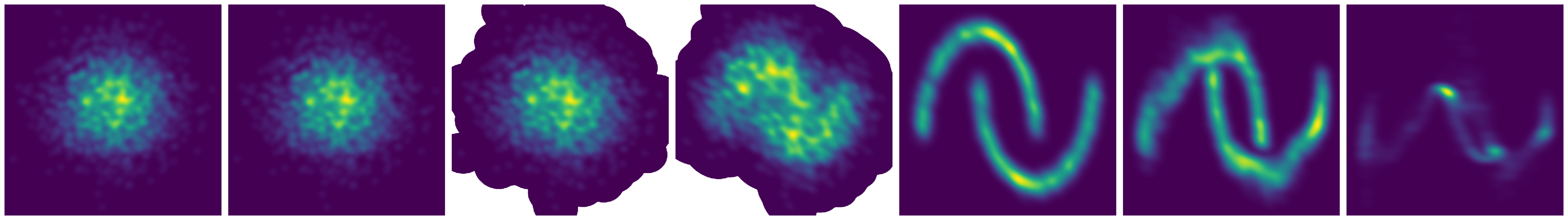}
    \caption{
        A longer time-frame depiction of OT-FM and Stable-FM flows in \cref{fig:dist-front}.
    }
    \label{fig:moons_dist_long}
\end{figure}

\begin{figure}[!ht]
    \centering
    \includegraphics[width=0.95\linewidth]{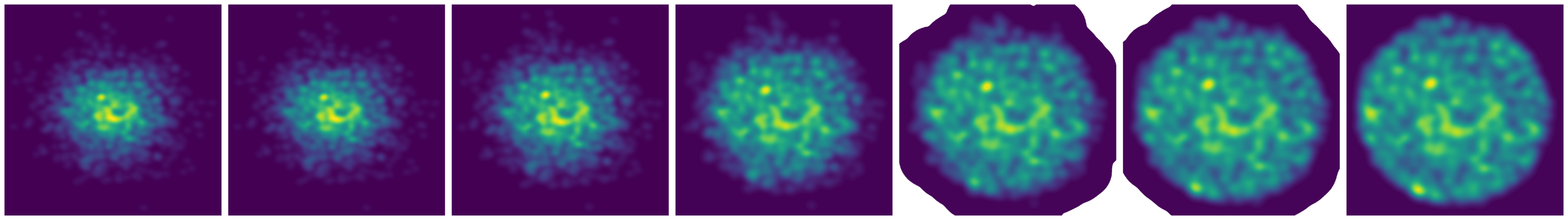}
    \includegraphics[width=0.95\linewidth]{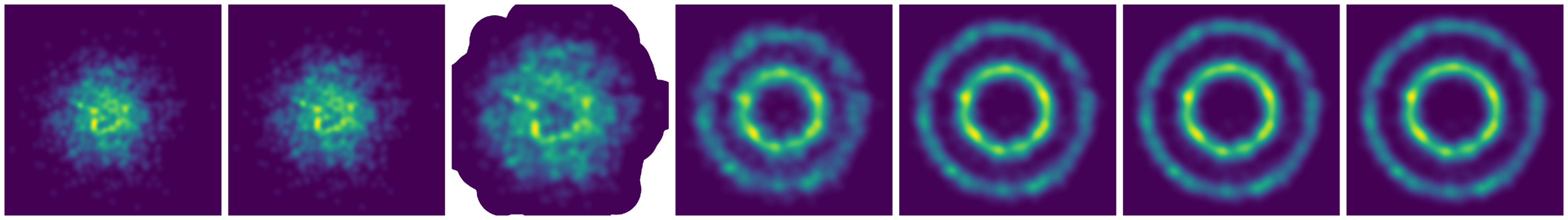}
    \includegraphics[width=0.95\linewidth]{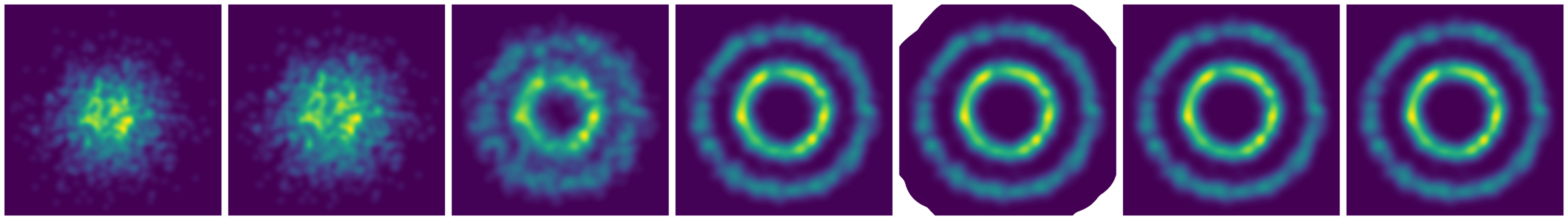}
    \includegraphics[width=0.95\linewidth]{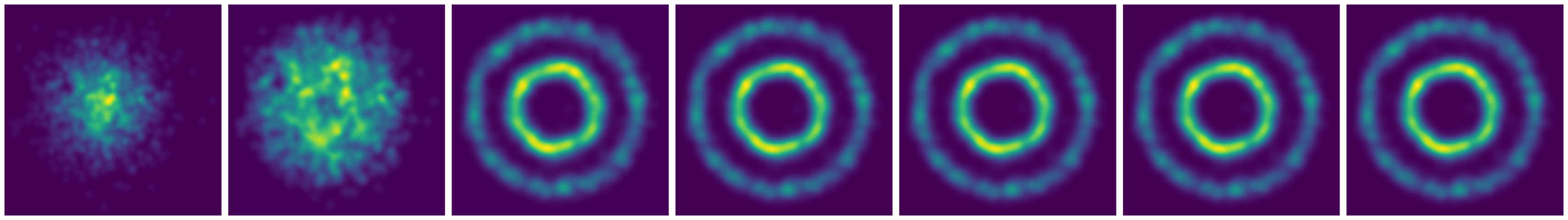}
    \includegraphics[width=0.95\linewidth]{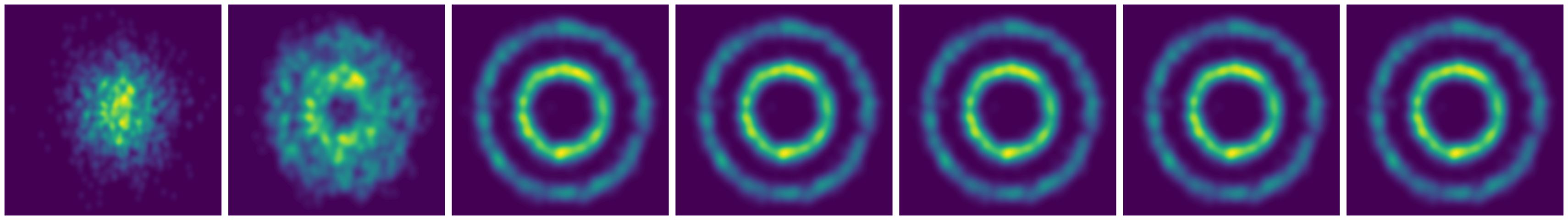}
    \includegraphics[width=0.95\linewidth]{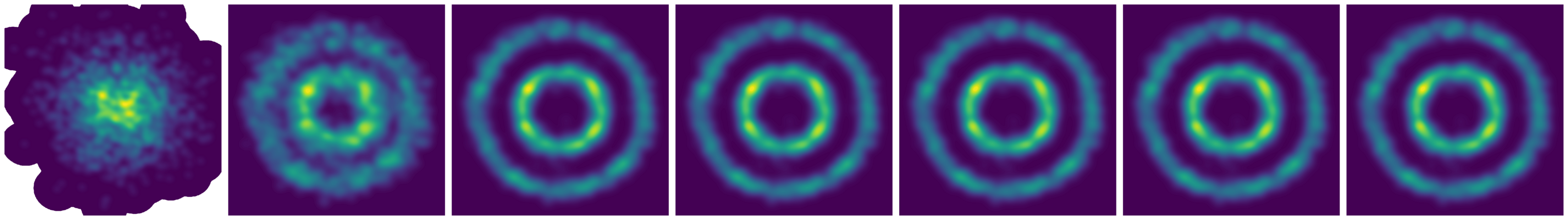}
    \caption{
        Flows of Stable-FM models trained with $L_\text{Auto}'$ from the standard normal distribution to the circles dataset with $\lambda_\tau = \ln(0.1)$ and 
        $\frac{\lambda_\mathbf{z}}{\lambda_\tau} \in \{1, 1.5, 2, 3, 3.5, 4\}$ (from top to bottom).
        Compare these values of $\frac{\lambda_\mathbf{z}}{\lambda_\tau}$ to those in \cref{fig:interpolation}.
        Each model has $4$ hidden layers with $500$ nodes and a softplus activation, including on the output (to enforce positive outputs).
    }
    \label{fig:stable-circles-dist}
\end{figure}

\begin{figure}[!ht]
    \centering
    \includegraphics[width=0.95\linewidth]{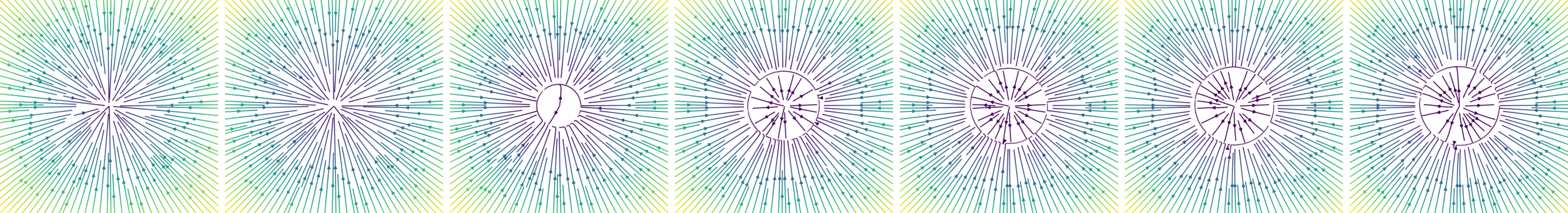}
    \includegraphics[width=0.95\linewidth]{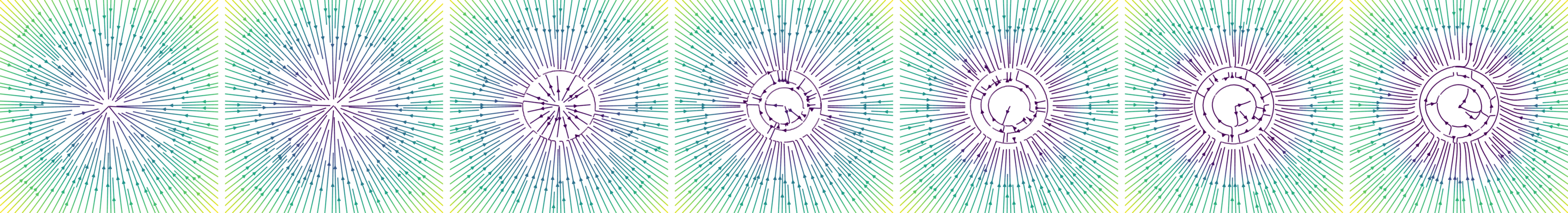}
    \includegraphics[width=0.95\linewidth]{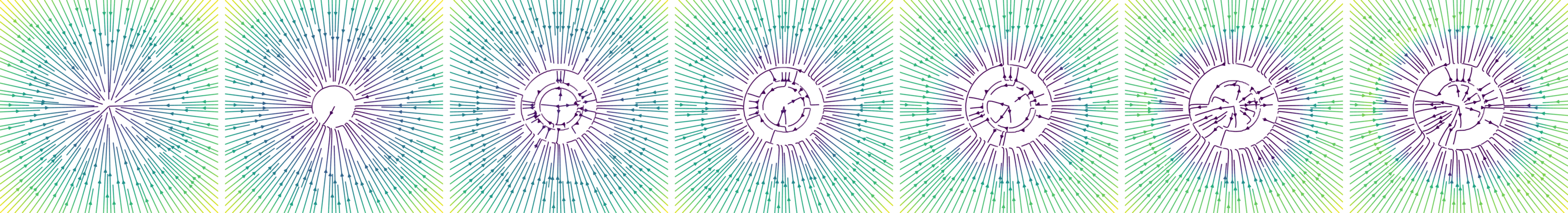}
    \includegraphics[width=0.95\linewidth]{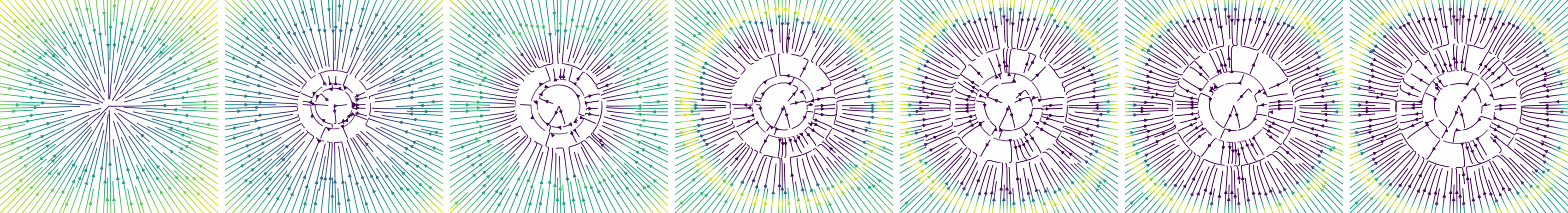}
    \includegraphics[width=0.95\linewidth]{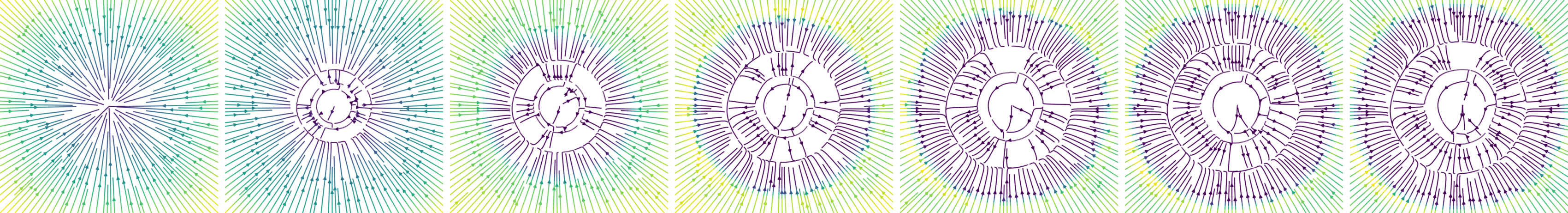}
    \caption{The vector fields corresponding to the flows in \cref{fig:stable-circles-dist}}
    \label{fig:stable-circles-vecs}
\end{figure}
\end{document}